\documentclass[journal,twoside,web]{ieeecolor}
\usepackage{algorithm}
\usepackage[noend]{algpseudocode}
\usepackage{amsfonts}
\usepackage{amsmath}
\usepackage{amssymb}
\usepackage{booktabs}
\usepackage[labelfont=bf,format=plain,justification=raggedright,singlelinecheck=false]{caption}
\usepackage{cite}
\usepackage{dsfont}
\usepackage{generic}
\usepackage{graphicx}
\usepackage{multirow}
\usepackage{tabularx}
\usepackage{textcomp}
\usepackage{tikzscale}
\usepackage{tkz-euclide}
\usepackage{stmaryrd}
\usepackage{subcaption}
\usepackage{xcolor}

\makeatletter
\let\NAT@parse\undefined
\makeatother
\usepackage{hyperref}
\hypersetup{
  colorlinks    = true, 
  urlcolor      = blue, 
  linkcolor     = blue, 
  citecolor     = blue 
}

\newcommand{\g}[3]{g_{#1}^{#2} \llbracket #3 \rrbracket}

\newcommand\newref[1]{#1\def\@currentlabel{#1}}
\newtheorem{definition}{Definition}
\newtheorem{theorem}{Theorem}[section]

\newtheorem{lemma}{Lemma}

\newtheorem{problem}{Problem}
\newtheorem{example}{Example}

\usetikzlibrary{shapes.geometric}
\tikzset{ac source/.style={
    circuit symbol lines,
    circuit symbol size = width 2 height 2,
    shape = generic circle IEC,
    /pgf/generic circle IEC/before background={
        \pgfpathmoveto{\pgfpoint{-0.8pt}{0pt}}
        \pgfpathsine{\pgfpoint{0.4pt}{0.4pt}}
        \pgfpathcosine{\pgfpoint{0.4pt}{-0.4pt}}
        \pgfpathsine{\pgfpoint{0.4pt}{-0.4pt}}
        \pgfpathcosine{\pgfpoint{0.4pt}{0.4pt}}
        \pgfusepath{stroke}
    },
    transform shape
}}

\definecolor{policyblue}{RGB}{31,119,180}
\definecolor{policyorange}{RGB}{255,127,14}
\definecolor{policygreen}{RGB}{44,160,44}
\definecolor{policyred}{RGB}{214,39,40}
\definecolor{policypurple}{RGB}{148,103,189}
\newcommand{\solidblueline}{\raisebox{2pt}{\tikz{\draw[-,policyblue,line width = 2pt](0,0) -- (2mm,0);}}}
\newcommand{\solidorangeline}{\raisebox{2pt}{\tikz{\draw[-,policyorange,line width = 2pt](0,0) -- (2mm,0);}}}
\newcommand{\solidgreenline}{\raisebox{2pt}{\tikz{\draw[-,policygreen,line width = 2pt](0,0) -- (2mm,0);}}}
\newcommand{\solidredline}{\raisebox{2pt}{\tikz{\draw[-,policyred,line width = 2pt](0,0) -- (2mm,0);}}}
\newcommand{\solidpurpleline}{\raisebox{2pt}{\tikz{\draw[-,policypurple,line width = 2pt](0,0) -- (2mm,0);}}}

\definecolor{nnblue}{RGB}{32,116,180}
\definecolor{nnorange}{RGB}{252,124,12}
\definecolor{sdtgreen}{RGB}{48,164,44}
\definecolor{sdtred}{RGB}{216,36,44}

\def\BibTeX{{\rm B\kern-.05em{\sc i\kern-.025em b}\kern-.08em T\kern-.1667em\lower.7ex\hbox{E}\kern-.125emX}}
\begin{document}
\title{Equivalent and Compact Representations of Neural Network Controllers With Decision Trees}

\author{%
Kevin~Chang,%
\,Nathan~Dahlin,\\
\,Rahul~Jain,~\IEEEmembership{Senior Member,~IEEE},%
\,~Pierluigi~Nuzzo,~\IEEEmembership{Senior Member,~IEEE}%
\thanks{K. Chang, R. Jain, and P. Nuzzo are with the Ming Hsieh Department of Electrical and Computer Engineering, University of Southern California, Los Angeles, USA. Email: {\tt\small \{kcchang, rahul.jain, nuzzo\}@usc.edu}. N. Dahlin was with the Department of Electrical and Computer Engineering, University of Illinois at Urbana-Champaign, Urbana, USA. He is now with the Department of Electrical Engineering, State University of New York at Albany, Albany, USA. Email: {\tt\small ndahlin@albany.edu}.}
}

\maketitle
\begin{abstract}
Over the past decade, neural network (NN)-based controllers have demonstrated remarkable efficacy in a variety of decision-making tasks.
However, their black-box nature and the risk of unexpected behaviors pose a challenge to their deployment in real-world systems requiring strong guarantees of correctness and safety.
We address these limitations by investigating the transformation of NN-based controllers into equivalent soft decision tree (SDT)-based controllers and its impact on verifiability. 
In contrast to existing work, we focus on discrete-output NN controllers including rectified linear unit (ReLU) activation functions as well as argmax operations.
We then devise an exact yet efficient transformation algorithm which automatically prunes redundant branches.
We first demonstrate the practical efficacy of the transformation algorithm applied to an autonomous driving NN controller within OpenAI Gym's CarRacing environment.
Subsequently, we evaluate our approach using two benchmarks from the OpenAI Gym environment. 
Our results indicate that the SDT transformation can benefit formal verification, showing runtime improvements of up to $21 \times$ and $2 \times$ for MountainCar-v0 and CartPole-v1, respectively.
\end{abstract}

\section{Introduction}
\label{sec:introduction}
\IEEEPARstart{O}{ver} the last decade, neural network (NN)-based techniques have exhibited outstanding efficacy in a variety of decision-making tasks and some of the most challenging control problems, possibly surpassing human-level performance~\cite{mnih2013playing, ge2013stable, lecun2015deep, wang2016does, taigman2014deepface, silver2016mastering}.
However, their deployment in safety-critical applications, such as autonomous driving and flight control raises concerns due to the black-box nature of NNs and the risk of unexpected behaviors~\cite{huang2017adversarial, naik2020robustness, xu2021fast, wang2021beta}.

Researchers have proposed distillation methods~\cite{hinton2015distilling, dahlin2022practical} as a way to overcome the limitations of NN-based controllers. Distillation transfers the learned knowledge and behaviors of NN controllers to more interpretable surrogate models, such as decision trees.
WIn principle, lossless transfers are possible, as the obtained efficacy of NN controllers versus other models is not necessarily due to their richer representative capacity, but rather to the many regularization techniques available to facilitate training~\cite{ba2014deep, bastani2018verifiable}. In practice, however, the distilled models typically fall short of the real-time performance of their full counterparts, and the identification of effective metrics to characterize the approximation quality of a distilled model versus the original NN is itself an open problem. Still, allowing for some loss in fidelity, by compressing NN controllers into simpler and more compact models, distillation can facilitate formal verification.

In this paper, we focus on the \emph{exact systematic transformation} of NN-based controllers into equivalent soft decision tree (SDT)-based controllers and the empirical evaluation of the impact of this transformation on controller verifiability.
As the policies of the original NN and distilled SDT provably agree in all states, the SDT models can be examined to expedite verification before the NN is deployed.
Moreover, unlike prior work, we consider discrete-output NN controllers with rectified linear unit (ReLU) activation functions and argmax operations. 
These discrete-action NNs are particularly challenging from a verification standpoint, as they tend to amplify the approximation errors generated by reachability analysis of the closed-loop control system. 
Our contributions can be stated as follows:
\begin{itemize}
    \item We first prove that any discrete output argmax-based NN controller has an \emph{equivalent} SDT in the sense that they enact identical policies over all states.
        This is done by presenting a constructive procedure to transform any NN controller into an equivalent SDT controller.

    \item We show that our constructive procedure for creating an equivalent SDT controller is computationally practical, in that the number of nodes in the output SDT scales polynomially with the maximum width of the NN hidden layers.
        To the best of our knowledge, this is the first such computationally practical transformation algorithm.  

    \item We first demonstrate the practical viability and efficacy of the proposed transformation on the OpenAI Gym~\cite{brockman2016openai} CarRacing environment, motivated by autonomous driving.
        We then empirically evaluate the relative computational efficiency of formally verifying the SDT surrogates instead of the original NN controllers for two OpenAI Gym environments. Results indicate a significant speed improvement with the SDT controller: 21 times faster in MountainCar-v0 and twice as fast in CartPole-v1 compared to the original NN controller.
\end{itemize}

\noindent Our results suggest that the presented SDT transformation can accelerate closed-loop NN control verification, and thus potentially benefit applications where performance guarantees are critical but deep learning methods are the primary choice for control design.

In a preliminary version of this work, we first explored the transformation of neural network-based controllers into equivalent soft decision tree (SDT)-based controllers and its impact on verifiability~\cite{chang2023exact}. In this article, we expand our investigation with a comprehensive complexity analysis. Moreover, we illustrate the successful transformation of an autonomous controller, with human-like vehicle navigation, into an SDT-based controller.

\section{Related Work} \label{sec:related_work}

\emph{Distillation}~\cite{hinton2015distilling} has been used to transfer the knowledge and behavior of NN-based controllers to other models in an approximate, e.g., data-driven, manner. 
It can be performed by training shallow NNs to mimic the behavior of state-of-the-art NNs using a teacher-student paradigm~\cite{ba2014deep}.
The distillation of SDTs from expert NNs was shown to lead to better performance than direct training of SDTs~\cite{hinton2015distilling}.
Furthermore, distillation has demonstrated success in reinforcement learning (RL) problems, where the \textsc{Dagger} algorithm is used to transfer knowledge from Q-value NN models via simulation episodes~\cite{ross2011reduction}.
Finally, distilling to SDTs has also been suggested as a means to explain the internal workings of black-box NNs~\cite{frosst2017distilling}.

In contrast, exact \emph{transformation} methods have been proposed to convert feedforward NNs into simpler models while preserving equivalence~\cite{lee2019oblique, sudjianto2020unwrapping, nguyen2020towards}.
Locally constant networks~\cite{lee2019oblique} and linear networks~\cite{sudjianto2020unwrapping} have been introduced as intermediate representations to establish the equivalence between NNs and SDTs whose worst-case size scales exponentially in the maximum width of the NN hidden layers.
Nguyen~et~al.~\cite{nguyen2020towards} proposed transforming NNs with ReLUs into decision trees and then compressing the trees via a learning-based approach.
As in previous approaches, our algorithm preserves equivalence with the SDTs.
However, it also guarantees, without the need for compression, that the size of the tree scales polynomially in the width of the maximum hidden layer of the NN.
Finally, we also provide quantitative evidence about the impact of the proposed transformations on the verifiability of the controllers. 

\section{Preliminaries} \label{sec:prelims}

We review key definitions for both neural networks and soft decision trees using Fig.~\ref{fig_example} as an illustrative example. Throughout the paper, we use $A^\top$ to represent the transpose of $A$, $A_{i,j}$ to denote the element in the $i$-th row and $j$-th column of matrix $A$, and $a_i$ or $(a)_i$ to denote the $i$-th element of vector $a$.
Additionally, we define the indicator function $\mathds{1}_{x \in D}$ as a function of $x$ which evaluates to $1$ if $x \in D$ and $0$ otherwise. 
All the notation used in this paper is listed in Appendix~\ref{app_notation}.

\subsection{Neural Networks} 

Let $L$ denote the total number of layers in a NN, $N^l$ denote the \emph{width}, or number of neurons in the $l$-th layer, and define $\mathbf{N}:=(N^1,\dots,N^L)$.
In Fig.~\ref{fig_ex_NN}, we have $L=3$, with a two-neuron input layer ($N^1 = 2$), a single-neuron hidden layer ($N^2 = 1$), and a three-neuron output layer ($N^3 = 3$), so that $\mathbf{N}=(2,1,3)$. 
The nodes in layer $l$, with $l\geq 2$, are fully connected to the previous layer.
Edges and nodes are associated with weights and biases, denoted by $W^l$ and $B^l$, respectively, for $l\geq 2$.
The output at each neuron is determined by its inputs using a \emph{feedforward function}, defined below, and then passed through a rectified linear unit (ReLU) \emph{activation function}.

\begin{definition} \label{def_ff}
\emph{(NN Layer Feedforward Function).}
Given input $x \in \mathbb{R}^{N^{l-1}}$, the feedforward function $f_F^l: \mathbb{R}^{N^{l-1}} \rightarrow \mathbb{R}^{N^l}$ maps $x$ to the output of layer $l$, with $l\geq 2$, where the output of the $i$-th neuron in layer $l$ is defined as
\begin{equation}\nonumber
(f_F^l(x))_i := \sum_{j=1}^{N^{l-1}} W_{i, j}^l x_j + B_i^l.
\end{equation}
\end{definition}
We also recall the definition of \emph{characteristic function}. Let $\alpha$ be the element-wise ReLU activation function, i.e., $\alpha(x)_i = \max\{x_i,0\}$. 

\begin{definition} \label{def_nn_l}
\emph{(NN Characteristic Functions).}
Given input $x \in \mathbb{R}^{N^1}$, the characteristic function at layer $l$ of NN $\mathcal{N}$, $f_\mathcal{N}^l: \mathbb{R}^{N^1} \rightarrow \mathbb{R}^{N^l}$, is defined as
$$f_\mathcal{N}^l(x) := f_F^l \circ \alpha \circ f_F^{l-1} \circ \dots \circ f_F^3 \circ \alpha \circ f_F^2(x).$$
Moreover, we consider an NN classifier that includes an argmax layer at the end. 
The characteristic function of the NN is then defined as
$$f_{\mathcal{N}}(x) := \underset{i \in \{1, \dots, N^L\}}{\arg\max} \{ (f_\mathcal{N}^L(x))_i \},$$ where we assume that, when the $\arg\max$ is not a singleton, a deterministic tie-breaking procedure is used to select a single index.
\end{definition}

\subsection{Soft Decision Trees}

\begin{figure*}[t]
    \centering
    \begin{subfigure}{0.45\textwidth}
        \includegraphics[width=\columnwidth]{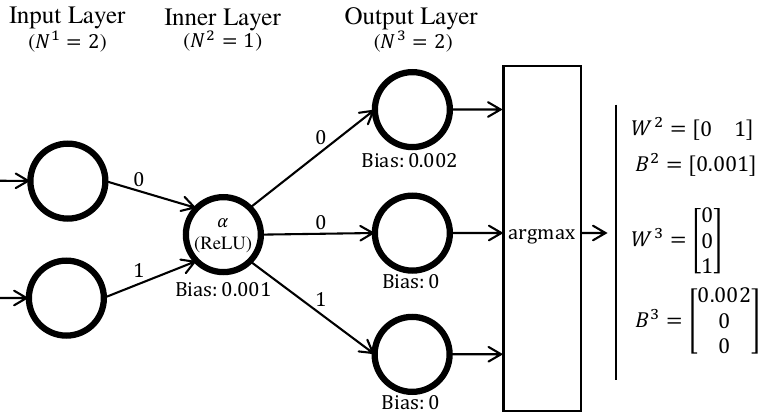}
    \caption{Reference NN.} \label{fig_ex_NN}
    \end{subfigure}
    \begin{subfigure}{0.5\textwidth}
        \includegraphics[width=\columnwidth]{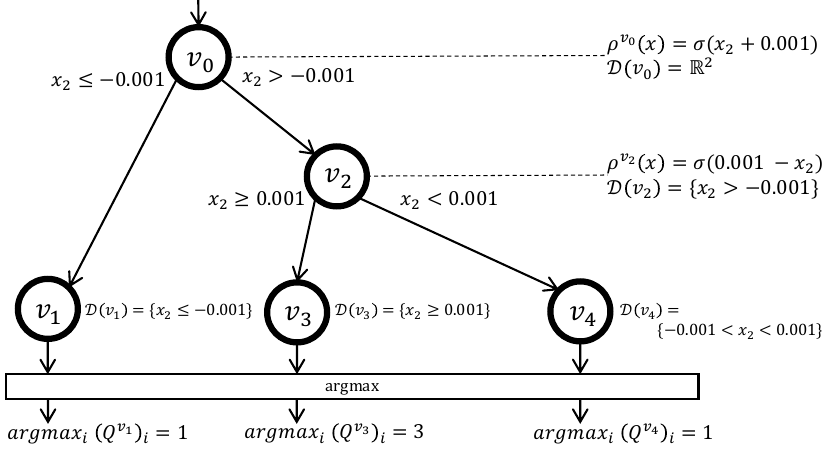}
    \caption{Equivalent SDT for the NN in Fig.~\ref{fig_ex_NN}.} \label{fig_ex_SDT}
    \end{subfigure}
    \caption{Example NN-to-SDT transformation ($\sigma(.)$ is the sigmoid function, $\{x_i < k\}$ is shorthand notation for $\{x_i: x_i < k\}$).}
    \label{fig_example}
\end{figure*}

While traditional decision trees arrive at output decisions based on the sequential comparison of individual features to threshold values, oblique decision trees~\cite{lee2019oblique}, multivariate decision trees~\cite{nguyen2020towards}, or SDTs \cite{irsoy2012soft} provide additional flexibility by enabling inner nodes to consider functions of multiple features.
Following on Frosst~\cite{frosst2017distilling}, we focus on SDTs that select paths from the root to the leaves based on binary decisions, rather than calculating the probability of selecting each branch and computing the distribution at each leaf based on the weighted sum of the branch distributions.  
Let $\mathcal{I}_{\mathcal{S}}$ and $\mathcal{L}_{\mathcal{S}}$ denote the sets of inner and leaf nodes for an SDT $\mathcal{S}$ with input dimension $n$.
Each inner node $v\in\mathcal{I}_{\mathcal{S}}$ is associated with weights $w^v\in\mathbb{R}^n$ and a bias term $b^v\in\mathbb{R}$.
At node $v$, $w^v$ and $b^v$ are applied to input $x$.
The resulting value is passed through an activation function $\sigma$, producing a scalar, $\rho^v(x) = \sigma(x^\top w^v + b^v)$, which is compared to a threshold to determine whether to proceed to the left or right branch.
The sigmoid function is utilized in training to mitigate hard decisions.
In this paper, while we focus on NN transformation rather than training, we adhere to the sigmoid function formulation for consistency with prior research.

Each leaf node $v\in\mathcal{L}_{\mathcal{S}}$ is associated with a vector $Q^v$.
In general, $Q^v$ is a distribution over possible output selections.
Rather than learning the parameters $(w^v,b^v)$ or $Q^v$ for each node $v$ from data as in the literature~\cite{frosst2017distilling}, we derive these parameters via transformation of a reference NN.
We denote the left and right child of an inner node $v$ as $l(v)$ and $r(v)$, respectively, and the parent of $v$ as $p(v)$.
The root node of the SDT is $v_0$.
We now define the SDT characteristic functions.

\begin{definition} \label{def_sdt_v} \emph{(SDT Characteristic Functions).}
Let $\mathcal{S}$ be an SDT with input dimension $n$.
Given input vectors $x \in \mathbb{R}^{n}$, the characteristic function of node $v$, $f_\mathcal{S}^v: \mathbb{R}^{n} \rightarrow \mathbb{N}$, is defined recursively as 
$$f_\mathcal{S}^v(x) = \left\{ \begin{array}{cl}
            \arg\max_k\{(Q^v)k\},  & \text{if } v\in\mathcal{L}_{\mathcal{S}},\\
        f_\mathcal{S}^{l(v)}(x), & \text{if } v\in\mathcal{I}_{\mathcal{S}} \wedge \rho^v(x) \leq 0.5, \\
        f_\mathcal{S}^{r(v)}(x), & \text{if } v\in\mathcal{I}_{\mathcal{S}} \wedge \rho^v(x) > 0.5.
\end{array} \right.$$
Finally, we define the SDT characteristic function as $$f_{\mathcal{S}}(x) := f_\mathcal{S}^{v_0}(x),$$ where $v_0$ denotes the root node.    
\end{definition}

The output of each leaf node of the SDT is then directly determined by a single path within the tree, guided by binary decisions~\cite{frosst2017distilling}. Throughout this work, we assume that NNs and SDTs use the same deterministic $\arg\max$ tie-breaking procedure. We further introduce the notation $\mathcal{D}(v)$ to represent the effective \emph{domain} of an SDT node $v$, i.e., the set of all possible   inputs arriving at $v$.
The set $\mathcal{D}(v)$ can be computed recursively in a top-down manner from the root node as follows
\begin{equation} \label{equ_rule}
    \mathcal{D}(v) = 
    \left\{ \begin{array}{ll}
        \mathbb{R}^{n}, & \text{if } v = v_0, \\
        \mathcal{D}(p(v)) \;\cap & \text{if } v \not = v_0 \;\wedge \\
        \quad \{x | \rho^{p(v)}(x) \leq 0.5\}, & \quad l(p(v)) = v, \\
        \mathcal{D}(p(v)) \;\cap & \text{if } v \not = v_0 \;\wedge \\
        \quad \{x | \rho^{p(v)}(x) > 0.5\}, & \quad r(p(v)) = v, \\
    \end{array} \right.
\end{equation}
where $n$ is the input dimension of the SDT.

For example, for the SDT in Fig.~\ref{fig_ex_SDT}, the inputs $x=(x_1,x_2)^\top \in \mathbb{R}^2$ are processed starting from the root node $v_0$, where $\mathcal{D}(v_0)=\mathbb{R}^2$.
The input range related to node $v_1$ is $\mathcal{D}(v_1) = \{x\in\mathbb{R}^2|x_2 \leq -0.001\}$, i.e., $v_1$ is reached only if $x_2 \leq -0.001$.
Similarly, node $v_2$ is reached when $x_2 > -0.001$, i.e., $\mathcal{D}(v_2) = \{x\in\mathbb{R}^2|x_2 > -0.001\}$.
In the remainder of the paper, we omit the reference to the underlying space in the expressions for the node input domains, when it is clear from the context, and simply write, e.g.,  $\mathcal{D}(v_1) = \{x_2\leq -0.001\}$ and $\mathcal{D}(v_2) = \{x_2>-0.001\}$. 

\begin{table*}
\caption{Rules for constructing an SDT node $v$.}
\label{tbl_transform}
\setlength\tabcolsep{1.5pt} 
\renewcommand{\arraystretch}{1.5} 
\begin{center} \resizebox{1.0\linewidth}{!}{
    \begin{tabular}{ l | l | l }
        \multicolumn{1}{c|}{\textbf{Rules}} & \multicolumn{1}{c|}{\textbf{Conditions}} & \multicolumn{1}{c}{\textbf{Actions}} \\
    \hline
    1: Split based on ReLU & (\hypertarget{eq_condition_1}{A})$\begin{array}{c} \exists \ i, l < L \ \text{s.t.} \ \mathcal{D}(v) \cap \{\g{i}{l}{v}(x) \leq 0\} \not = \emptyset \;\wedge \\ \mathcal{D}(v) \cap \{\g{i}{l}{v}(x) > 0\} \not = \emptyset \end{array}$ & $\rho^v(x) = \left\{ \begin{array}{l}
\sigma(\g{i}{l}{v}(x)) \\ \text{for} \ i, \min_{l \in \{1, \dots, L\}}\,l \ \text{such that (\hyperlink{eq_condition_1}{A}) holds} \\ \end{array} \right.$ \\ [8pt]
    \hline
    2: Split based on output layer & (\hypertarget{eq_condition_2}{B})$\begin{array}{c} \text{(\hyperlink{eq_condition_1}{A}) is not satisfied and} \ \exists \ i_1, i_2 \ \text{s.t.} \\ \mathcal{D}(v) \cap \bigcap_{i' \neq i_1} \{\g{i_1}{L}{v}(x) \geq \g{i'}{L}{v}(x)\} \not = \emptyset \;\wedge \\ \mathcal{D}(v) \cap \bigcap_{i' \neq i_2} \{\g{i_2}{L}{v}(x) > \g{i'}{L}{v}(x)\} \not = \emptyset \end{array}$ & $\rho^v(x) = \sigma(\g{i_1}{L}{v}(x) - \g{i_2}{L}{v}(x))$ \\ [15pt]
    \hline
    3: Form leaf node & (\hypertarget{eq_condition_3}{C}) \text{Neither (\hyperlink{eq_condition_1}{A}) nor (\hyperlink{eq_condition_2}{B}) are satisfied} & $(Q^v)_i = \left\{\begin{array}{ll}
            1 & \text{if } i=  \arg\max_k \{\g{k}{L}{v}(x)\} \ \forall x \in \mathcal{D}(v) \\
        0 & \text{otherwise} \\
        \end{array}\right.$ \\
    \end{tabular}
} \end{center}
\end{table*}

\section{Transformation from Neural Network to Decision Tree} \label{sec:transformation}

We present an algorithm for constructing an SDT $\mathcal{S}_{\mathcal{N}}$ which is equivalent to a reference NN $\mathcal{N}$ in the sense that $f_{\mathcal{N}}(x)=f_{\mathcal{S}_{\mathcal{N}}}(x)$ for all inputs $x\in\mathbb{R}^n$. We elucidate the connection between fully connected NNs with ReLU activation and SDTs via the example in Fig.~\ref{fig_example}, showing that there exists an SDT $\mathcal{S}_\mathcal{N}$ equivalent to a such reference NN $\mathcal{N}$.

\begin{example} \label{example_simple_nn}
Given the NN $\mathcal{N}$ and SDT $\mathcal{S}_\mathcal{N}$ in Fig.~\ref{fig_example}, we demonstrate their equivalence for all inputs $x \in \mathbb{R}^n$, specifically, $f_{\mathcal{N}}(x)=f_{\mathcal{S}_{\mathcal{N}}}(x)$. In the NN of Fig.~\ref{fig_ex_NN}, the output of the neuron in layer 2 following the ReLU activation is 
\begin{equation}\label{eq:ex_relu_split}
\begin{aligned}
& \alpha(f^2_{\mathcal{N}}(x)) = \alpha\left(W^2x + B^2\right) \\
& = \alpha(x_2+0.001) = \begin{cases}
    0, &\text{if }x_2\leq -0.001,\\
    x_2+0.001, &\text{if }x_2>-0.001.
\end{cases}
\end{aligned}
\end{equation}
Carrying these two cases forward through the output layer $L=3$, we have 
\begin{equation*}
\begin{aligned}
& f^3_{\mathcal{N}}(x) = W^3 \alpha(f^2_\mathcal{N}(x)) + B^3 \\
& = \begin{cases} (0.002,\,0,\,0)^{\top},&\text{if }x_2\leq -0.001,\\
(0.002,\,0,\,x_2+0.001)^{\top},&\text{if }x_2>-0.001.
\end{cases}
\end{aligned}
\end{equation*}
Thus, when $x_2\leq -0.001$, we have \begin{equation*}
f_{\mathcal{N}}(x) = \underset{i \in \{1, 2, 3\}}{\arg\max} \{(f_\mathcal{N}^3(x))_i\} = 1,
\end{equation*}
while when $x_2>-0.001$, we have 
\begin{equation}\label{eq:ex_output_split}
f_{\mathcal{N}}(x) = \begin{cases}
1,&\text{if }x_2\leq 0.001,\\
3,&\text{if }x_2> 0.001,
\end{cases}
\end{equation}
assuming we take the lowest index to break ties.

Turning to the SDT in Fig.~\ref{fig_ex_SDT}, we observe that \eqref{eq:ex_relu_split} partitions the input space into two subsets based on the ReLU activation function of the inner neuron.
This is precisely the split that occurs at the root node $v_0$ of the SDT. Equation~\eqref{eq:ex_output_split} further partitions one of these subsets based on the output layer values. This split occurs at $v_2$, forming the domains of leaf nodes $v_3$ and $v_4$. Finally, we obtain three regions, in which $f_{\mathcal{N}}(x)$ is constant, corresponding to leaf nodes $v_1$, $v_3$, and $v_4$ in Fig.~\ref{fig_ex_SDT}. Overall, we have 
\begin{equation*}
f_{\mathcal{S}_\mathcal{N}}(x) = \begin{cases}
1,&\text{if }x_2\leq 0.001,\\
3,&\text{if }x_2> 0.001.
\end{cases}
\end{equation*}
\end{example}

As shown in \eqref{eq:ex_relu_split}-\eqref{eq:ex_output_split} of Example~\ref{example_simple_nn}, for subsets of the NN input space corresponding to the two possible regimes of the ReLU activation in layer two, $f^L_{\mathcal{N}}(x)$ is an affine function. Such subsets may be further partitioned based on the argmax operation, as in \eqref{eq:ex_output_split}. The operation of fully connected NNs with ReLU activation and argmax output can then be understood in terms of a successive assignment of the inputs to increasingly refined space subsets. At each layer, these subsets can be shown to be convex polyhedra~\cite{lee2019oblique}. The identification of SDT splits and leaf node assignments based on NN neuron activation functions and output values forms the core of our transformation technique, which we further explain below. 

\subsection{Pre- and Post-Activation Functions}

We establish the relationship between $\mathcal{N}$ and $\mathcal{S}_{\mathcal{N}}$ by first introducing, for each node $v$ of $\mathcal{S}_{\mathcal{N}}$, \textit{pre-activation functions} of the form $\g{i}{l}{v}\,:\,\mathcal{D}(v)\to \mathbb{R}$ and \textit{post-activation functions} of the form $\overline{\g{i}{l}{v}}\,:\,\mathcal{D}(v)\to \mathbb{R}$.
A pre-activation function provides the output of the $i$-th neuron in the $l$-th layer of $\mathcal{N}$ as a function of the input in $\mathcal{D}(v)$ prior to the ReLU. A post-activation function gives the neuron output after the ReLU, i.e., $\overline{\g{i}{l}{v}} = \alpha(\g{i}{l}{v})$.
The pre-activation functions for $v$ are recursively defined for $i=1,\dots,N^l$ and $l=2,\dots,L$ as
\begin{equation} \label{equ_pre}
\resizebox{\ifdim\width>0.9\linewidth 0.9\linewidth \else \width \fi}{!}{$
\begin{aligned}
& \g{i}{l}{v}(x) = \\
& \begin{cases}
    \sum_{j=1}^{N^1} W^2_{i,j} x_j + B^2_i, & \text{if } l = 2, \\
    \sum_{j=1}^{N^{l-1}}W^l_{i,j} \overline{\g{j}{l-1}{v}}(x) + B_i^l,  & \text{if } l > 2 \wedge \overline{\g{j}{l-1}{v}}(x) \neq \bot \\
    \quad & \quad \text{for } j =1, 2, \dots, N^{l-1}, \\
    \bot, & \text{otherwise}.
\end{cases}
\end{aligned} 
$}
\end{equation}
If the output of the $i$-th neuron in the $l$-th layer is not an affine function over the domain $\mathcal{D}(v)$ and cannot be represented by an appropriate pre-activation function $\g{i}{l}{v}(x)$, we set $\g{i}{l}{v}(x) = \bot$, where $\bot$ denotes `undefined.' 

The post-activation functions for node $v$ are recursively defined for $i=1,\dots,N^l$ and $l=2,\dots,L-1$ as
\begin{equation} \label{equ_post}
\begin{split}
\resizebox{\ifdim\width>0.9\linewidth 0.9\linewidth \else \width \fi}{!}{$
\overline{\g{i}{l}{v}}(x) = \begin{cases}
    \g{i}{l}{v}(x), & \text{if } \g{i}{l}{v} \not = \bot \;\wedge\\
                   & \quad \mathcal{D}(v) \cap \{\g{i}{l}{v}(x) \leq 0\} = \emptyset, \\
    0, & \text{if } \g{i}{l}{v} \not = \bot \;\wedge \\
      & \quad \mathcal{D}(v) \cap \{\g{i}{l}{v}(x) > 0 \} = \emptyset, \\
    \bot, & \text{otherwise}.
\end{cases}
$}
\end{split}
\end{equation}
Finally, we define $\g{}{l}{v} := (\g{1}{l}{v},\dots,\g{N^l}{l}{v})^{\top}$ and $\overline{\g{}{l}{v}}$ similarly, for $l=2,\dots, L$.
The pre-activation functions play a crucial role in determining the pre-threshold branching functions $\rho^v$ at each inner node $v$ of $\mathcal{S}_{\mathcal{N}}$.

\subsection{SDT Split and Leaf Formation Rules}

We construct the branches of $\mathcal{S}_{\mathcal{N}}$ based on the rules in Table \ref{tbl_transform}. Beginning with the root node $v_0$, for each node $v$ in $\mathcal{S}_{\mathcal{N}}$, we sequentially apply the rules in Table~\ref{tbl_transform} to obtain $\rho^v(x)$ or $Q^v$ in a top-down manner, from root to leaves. Specifically, for each node, we first check condition (\hyperlink{eq_condition_1}{A}); if condition (\hyperlink{eq_condition_1}{A}) is not applicable, we then proceed to  (\hyperlink{eq_condition_2}{B}), and subsequently to (\hyperlink{eq_condition_3}{C}). The details of these rules are explained below.

\subsubsection{Rule 1 (Split Based on ReLU)}

If (\hyperlink{eq_condition_1}{A}) holds for a neuron $i$ in layer $l$ of $\mathcal{N}$, then by \eqref{equ_post}, the post-activation function $\overline{\g{i}{l}{v}}$ is undefined. Therefore, we introduce a split at node $v$ in $\mathcal{S}_{\mathcal{N}}$ so that the post-activation function $\overline{\g{i}{l}{v}}$ associated with both the children $l(v)$ and $r(v)$ is defined and affine. Specifically, $\mathcal{D}(v)$ is partitioned by the hyperplane $\g{i}{l}{v}(x)=0$, resulting in $\mathcal{D}(l(v)) = \mathcal{D}(v)\cap\{\g{i}{l}{v}(x)\leq0\}$ and $\mathcal{D}(r(v))=\mathcal{D}(v)\cap\{\g{i}{l}{v}(x)>0\}$.
For example, for node $v_0$ in Fig.~\ref{fig_ex_SDT}, we create a split at $v_0$ with $\rho^{v_0}(x)=\sigma(\g{1}{2}{v_0}(x))$.

\subsubsection{Rule 2 (Split Based on Output Layer)} 

Suppose no partitions can be found for $\mathcal{D}(v)$ 
according to (\hyperlink{eq_condition_1}{A}). If, however, (\hyperlink{eq_condition_2}{B}) holds for neurons $i_1$ and $i_2$ in layer $L$ of $\mathcal{N}$, then $\arg\max_k \{\g{k}{L}{v}\}(x)$ is not a singleton in $\mathcal{D}(v)$. In this case, $\mathcal{D}(v)$ should be further partitioned based on the output layer values.
Consequently, we form a split at node $v$ where $\mathcal{D}(v)$ is partitioned by the hyperplane $\g{i_2}{L}{v}(x) - \g{i_1}{L}{v}(x) = 0$.
For example, for node $v_2$ in Fig.~\ref{fig_ex_SDT}, a split is created at $v_2$ with $\rho^{v_2}(x)$ defined as $\sigma(\g{1}{3}{v_2}(x) - \g{3}{3}{v_2}(x))$.

\subsubsection{Rule3 (Form Leaf Node)}

If neither (\hyperlink{eq_condition_1}{A}) nor (\hyperlink{eq_condition_2}{B}) holds, 
then the index of the neuron with the maximum output at layer $L$, i.e., the outcome of the $\arg\max$ operator, remains constant over the entire set $\mathcal{D}(v)$. We then declare $v$ as a leaf node, setting $(Q^v)_i=1$ if the neuron index $i$ corresponds to the largest output, and $0$ otherwise. For example, for node $v_1$ in Fig.~\ref{fig_ex_SDT}, we set $Q_1^{v_1}(x)=1$ and $Q_2^{v_1}(x)=Q_3^{v_1}(x)=0$.

The conditions above, e.g., $\mathcal{D}(v) \cap \{\alpha(\g{i}{l}{v}(x)) = 0\}\neq \emptyset$ in (\hyperlink{eq_condition_1}{A}), can be formulated and efficiently checked as feasibility problems for linear programs.

\begin{figure}[t]
    \centering
    \includegraphics[width=\ifdim\columnwidth>0.6\textwidth 0.7\columnwidth \else 1.0\columnwidth \fi]{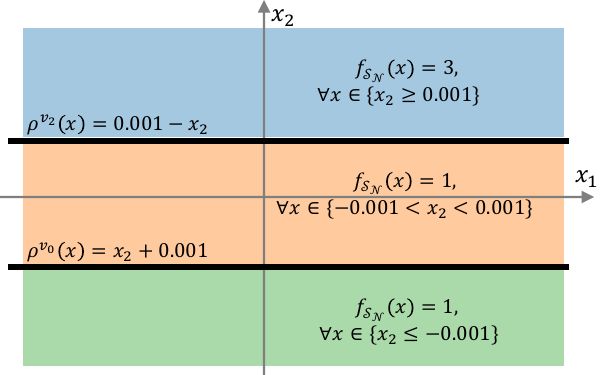}
    \caption{Characteristic function of the transformed SDT, $f_{\mathcal{S}_\mathcal{N}}(x)$, for Example~\ref{example_simple_nn}.}
    \label{fig_procedure}
\end{figure}

\subsection{Transformation Example} \label{sec:transformation_example}

We introduce our transformation procedure using the example in Fig.~\ref{fig_example}.
Assuming that the neural network in Fig.~\ref{fig_ex_NN} takes as input $x = (x_1, x_2)^{\top}\in\mathbb{R}^2$, for node $v_0$, we set $\mathcal{D}(v_0) = \mathbb{R}^2$. Using \eqref{equ_pre}, we can compute $\g{1}{2}{v_0}(x)$ as $x_2$.
Because neither $\mathbb{R}^2 \cap \{x_2 \leq -0.001\}$ nor $\mathbb{R}^2 \cap \{x_2 > -0.001\}$ is an empty set, we create a split at $v_0$ by applying (\hyperlink{eq_condition_1}{A}) from Table~\ref{tbl_transform}. This results in two children, $l(v_0)=v_1$ and $r(v_0)=v_2$, and a splitting function $\rho^{v_0}(x)=\sigma(\g{1}{2}{v_0}(x)) = x_2 + 0.001$, as pictorially represented in Fig.~\ref{fig_procedure}.

For the tree node $v_1$, we have $\mathcal{D}(v_1) = \mathbb{R}^2\cap \{x_2\leq-0.001\}$ and $\g{1}{2}{v_1}(x) = x_2+0.001$. Thus, we obtain $\overline{\g{1}{2}{v_1}}(x)=0$ by $\eqref{equ_post}$, and $\g{i}{3}{v_1}(x) = B^3_i$, that is, the output of the inner ReLU for any $x\in\mathcal{D}(v_1)$ is zero and the input to the final $\arg\max$ layer is simply the bias vector $\g{}{3}{v_1}(x) = B^3 = (0.002,0,0)^{\top}$. (\hyperlink{eq_condition_1}{A}) does not apply because $\mathcal{D}(v_1)\cap \{\g{1}{2}{v_1}(x)>0\}= \emptyset$, nor does (\hyperlink{eq_condition_2}{B}), because $\g{1}{3}{v_1}(x)>\g{i}{3}{v_1}(x)$ for any $i\in\{2,3\}$ and $x\in\mathcal{D}(v_1)$. Therefore, (\hyperlink{eq_condition_3}{C}) is applied with $Q_1^{v_1}(x)=1$ and $Q_2^{v_1}(x)=Q_3^{v_1}(x)=0$. As shown in Fig.~\ref{fig_procedure}, we have $f_{\mathcal{S}_\mathcal{N}}(x) = 1, \forall x \in \mathcal{D}(v_1) = \{x_2\leq-0.001\}$.

For node $v_2$, with domain $\mathcal{D}(v_2) = \mathbb{R}^2\cap\{x_2>-0.001\}$, using \eqref{equ_pre} and \eqref{equ_post}, we have $\g{1}{2}{v_2}(x) = \overline{\g{1}{2}{v_2}}(x) = x_2+0.001$, and $\g{}{3}{v_2}(x) = (0.002, 0, x_2+0.001)^\top$. (\hyperlink{eq_condition_1}{A}) cannot be satisfied since $\mathcal{D}(v_2) \cap \{\g{1}{2}{v_2}(x) \leq  0 \}$ is empty.
However, both 
$$\begin{aligned}
& \mathcal{D}(v_2) \cap \bigcap_{i' \neq 1}\{\g{1}{3}{v_2}(x) \geq \g{i'}{3}{v_2}(x)\} \\
& = \{-0.001 < x_2 \leq  0.001\}
\end{aligned}$$
and
$$\mathcal{D}(v_2) \cap \bigcap_{i' \neq 3}\{\g{3}{3}{v_2}(x) > \g{i'}{3}{v_2}(x)\} = \{0.001 < x_2 \}$$ are non-empty. Therefore, we use (\hyperlink{eq_condition_2}{B}) to set $\rho^{v_2}(x)$ as $\sigma(\g{1}{3}{v_2}(x) - \g{3}{3}{v_2}(x)) = 0.001 - x_2$, as shown in Fig.~\ref{fig_procedure}. We repeat this process for nodes $v_3$ and $v_4$ and transform the NN in Fig.~\ref{fig_ex_NN} to the SDT in Fig.~\ref{fig_ex_SDT}.

We illustrate the potential advantage of using the SDT controller over the NN controller in a verification problem by comparing the encoding of the characteristic functions of a NN $\mathcal{N}$ and its equivalent SDT $\mathcal{S}_{\mathcal{N}}$ using the example in Fig.~\ref{fig_example}. 
The encoding of $u = f_\mathcal{N}(x)$ includes constraints derived from both the encoding of $z = f_\mathcal{N}^2(x) = \alpha(x_2+0.001)$ and $u = \arg\max (0.002, 0, z)^T$. These are expressed as 
\begin{equation} \nonumber
\begin{array}{ll}
& \underbrace{(z \geq x_2+0.001) \wedge (z \geq 0) \wedge [(z = 0) \vee (z = x_2+0.001)]}_{z = \alpha(x_2+0.001)} \\
&\qquad\wedge \underbrace{[(u = 3) \vee (z \leq 0.002)] \wedge [(u = 1) \vee (z > 0.002)]}_{u = \arg\max (0.002, 0, z)^T},
\end{array}
\end{equation}
where $z$ is an auxiliary variable associated with the inner neuron. On the other hand, the encoding of $u = f_{\mathcal{S}_\mathcal{N}}(x) = \arg\max (0.001, 0, x_2)^T$ is given by
\begin{equation} \nonumber
[(u = 3) \vee (x_2 \leq 0.001)] \wedge [(u = 1) \vee (x_2 > 0.001)]. 
\end{equation}
We observe that the encoding of $f_{\mathcal{N}}(x)$ tends to produce longer formulas including a number of auxiliary variables that increases with the number of layers. Conversely, in SDTs, each node $v$ produces a scalar function $\rho^v(x)$ and the relation between the input $x$ and the output $f_{\mathcal{S}_\mathcal{N}}(x)$ is more straightforward, leading to a comparatively smaller expression, which helps reduce the verification time.

\subsection{Transformation Procedure}

We construct $\mathcal{S}_{\mathcal{N}}$ by starting with the root node $v_0$ and building the binary tree downwards.
Initially, we create the left branches, progressing until a leaf node is encountered.
Subsequently, we backtrack up the tree, defining the unexplored right branches to complete the construction of $\mathcal{S}_{\mathcal{N}}$.
A recursive method implementing this procedure is presented in Algorithm~\ref{alg_transformation}.
The transformation procedure effectively identifies a partition of the input space according to the domains associated with the SDT leaf nodes. Within these sets, the neural network characteristic function $f_{\mathcal{N}}$ is constant, and due to our choice of $Q^v$ at each leaf node $v$, we have that $f_{\mathcal{N}}(x) = f_{\mathcal{S}_{\mathcal{N}}}(x)$ over $\mathcal{D}(v)$.

\begin{algorithm}[t]
\caption{NN-to-SDT Transformation $\emph{NN2SDT}(\cdot)$}
\label{alg_transformation}
{\small
\begin{algorithmic}[1]
    \State \textbf{Global Variables: } NN parameters $L$, $\{N^l\}_{l=1}^L$, $\{W^l,b^l\}_{l=2}^L$
    \State \textbf{Input: } $v$    
    \State Compute $\mathcal{D}(v)$ via \eqref{equ_rule}
    \For{$l=2,\dots,L$}
    \For{$i=1,\dots,N^l$}
    \State Compute $\g{i}{l}{v}$, $\overline{\g{i}{l}{v}}$ via \eqref{equ_pre}, \eqref{equ_post}
        \EndFor
    \EndFor
    \If{ 
    ($\mathcal{D}(v),\g{i}{l}{v}$) satisfy (\hyperlink{eq_condition_1}{A})}
    \State $\rho^v(x) = \sigma(\g{i}{l}{v}(x))$, create $l(v)$, $r(v)$
    \State $\emph{NN2SDT}(l(v))$ \algorithmiccomment{{Recursion}}
    \State $\emph{NN2SDT}(r(v))$
    \ElsIf{
    ($\mathcal{D}(v),\g{i_1}{L}{v},\g{i_2}{L}{v}$) satisfy (\hyperlink{eq_condition_2}{B})}
    \State $\rho^v(x) = \sigma(\g{i_1}{L}{v}(x)-\g{i_2}{L}{v}(x))$, create $l(v)$, $r(v)$.
    \State $\emph{NN2SDT}(l(v))$ \algorithmiccomment{{Recursion}}
    \State $\emph{NN2SDT}(r(v))$
    \Else
        \State $(Q^v)_i =$  $\left\{\begin{array}{ll}
            1 & \text{if } i = \arg\max_k \{\g{k}{L}{v}\} \\
            0 & \text{otherwise} \\
        \end{array}\right.$
    \EndIf
\end{algorithmic}
}
\end{algorithm}

Prior to introducing our main equivalence result, we present two lemmas.
As \eqref{equ_pre} and \eqref{equ_post} show, the SDT pre- and post-activation functions, respectively, can, in general, become undefined. Therefore, Lemma~\ref{lemma_rule1} states that SDTs constructed via Algorithm \ref{alg_transformation} are always well-defined. 
\begin{lemma}\label{lemma_rule1}
For all inner nodes $v$ of $\mathcal{S}_{\mathcal{N}}$ the assignment $\rho^v = \sigma(\bot)$ can never occur. 
\end{lemma}
\begin{proof}
The assignment $\rho^v = \sigma(\bot)$ can only occur for a node $v$ satisfying (\hyperlink{eq_condition_1}{A}) or (\hyperlink{eq_condition_2}{B}). If (\hyperlink{eq_condition_1}{A}) applies, let $\rho^v = \sigma(\g{i}{l}{v})$ with $\g{i}{l}{v}=\bot$.
Then, due to (\ref{equ_pre}), there must also exist $l' < l$ and $i'$ such that $\g{i'}{l'}{v}$ is defined but $\overline{\g{i'}{l'}{v}}$ is not.
By (\ref{equ_post}), if $\overline{\g{i'}{l'}{v}}=\bot$, then both $\mathcal{D}(v) \cap \{\g{i'}{l'}{v} (x) \leq 0\}$ and $\mathcal{D}(v) \cap \{\g{i'}{l'}{v} (x) > 0\}$ are non-empty. In this case, (\hyperlink{eq_condition_1}{A}) would assign $\rho^v=\sigma(\g{i'}{l'}{v})$, which contradicts our assumption that $\rho^v = \sigma(\g{i}{l}{v})$.

If (\hyperlink{eq_condition_2}{B}) applies, then $\rho^v = \sigma(\g{i_1}{L}{v}-\g{i_2}{L}{v})=\sigma(\bot)$, and for some $i\in\{i_1,i_2\}$, we have 
$\g{i}{L}{v}=\bot$. We can then take $l=L$ and repeat the argument as for Rule 1 to arrive at a contradiction with the assignment $\rho^v=\sigma(\g{i_1}{L}{v}-\g{i_2}{L}{v})$.
\end{proof}

We now introduce Lemma~\ref{lemma_rule3}, which directly relates the pre-activation function $\g{}{L}{v}(x)$ for layer $L$ to the characteristic function $f^L_{\mathcal{N}}(x)$. 

\begin{lemma}\label{lemma_rule3}
Given a NN $\mathcal{N}$ and the SDT $\mathcal{S}_{\mathcal{N}}$ obtained from Algorithm~\ref{alg_transformation}, let $\mathcal{L}$ denote the set of leaf nodes of $\mathcal{S}_{\mathcal{N}}$. Then, for all $v\in\mathcal{L}$ and $x\in\mathcal{D}(v)$, we have 
$$\g{}{L}{v}(x)=f^L_{\mathcal{N}}(x).$$
\end{lemma}

\begin{proof} 
We prove this result via induction on $l=2,\dots,L$.
For $l=2$, for each node $v$ and $x\in\mathcal{D}(v)$, \eqref{equ_pre} gives $$\g{}{2}{v}(x)= W^2 x + B^2 = f_\mathcal{N}^2(x) .$$ 
For the induction step, assume that $(f_\mathcal{N}^l(x))_i = \g{i}{l}{v}(x)$ for all $i \in \{1, 2, \dots, N^l\}$. Then, by Definition \ref{def_ff}, for $i\leq N^{l+1}$ we have, for all $x\in\mathcal{D}(v)$, 
\begin{equation}\label{eq:lem3_NN_FF}(f_\mathcal{N}^{l+1}(x))_i = \sum_{j = 0}^{N^l} (W^{l+1}_{i,j} \alpha(\g{j}{l}{v}(x))) + B_i^{l+1}.\end{equation}
By Lemma~\ref{lemma_rule1}, we consider the possibilities corresponding to the first two cases in \eqref{equ_post}. First, if $\mathcal{D}(v) \cap \{\g{j}{l}{v}(x) \leq 0\}$ is empty, then $\g{j}{l}{v}(x) > 0$ for all $x \in \mathcal{D}(v)$, which implies 
\begin{equation}\label{eq:lem3_relu1}
    \alpha(\g{j}{l}{v}(x)) = \g{j}{l}{v}(x) = \overline{\g{j}{l}{v}}(x).
\end{equation}
Second, if $\mathcal{D}(v) \cap \{\g{j}{l}{v}(x) > 0\}$ is empty, then for all $x \in \mathcal{D}(v)$, we have $\g{j}{l}{v}(x) \leq 0$, which implies \begin{equation}\label{eq:lem3_relu2}
    \alpha(\g{j}{l}{v}(x)) = 0 = \overline{\g{j}{l}{v}}(x).
    \end{equation}
Therefore, combining \eqref{equ_pre} and \eqref{eq:lem3_NN_FF}-\eqref{eq:lem3_relu2}, we have 
$$(f^{l+1}_{\mathcal{N}}(x))_i = \sum_{j = 0}^{N^l} (W^{l+1}_{i,j} \overline{\g{j}{l}{v}(x)}) + B_i^{l+1}= \g{i}{l+1}{v}(x),$$
which concludes our proof. 
\end{proof}

We are now ready to state our main result on the equivalence of $\mathcal{N}$ and $\mathcal{S}_{\mathcal{N}}$.

\begin{theorem}\label{theorem}
Let $\mathcal{N}$ be a NN with input space $\mathbb{R}^n$ and $\mathcal{S}_{\mathcal{N}}$ the SDT resulting from the application of Algorithm~\ref{alg_transformation} to $\mathcal{N}$. 
Then, the corresponding characteristic functions are pointwise equal, i.e., 
$f_{\mathcal{N}}(x) = f_{\mathcal{S}_{\mathcal{N}}}(x)$ for all $x\in\mathbb{R}^n.$
\end{theorem}

\begin{proof}
By construction, the leaf node regions $\mathcal{D}(v)$ for $v\in\mathcal{L}$ form a partition of the input space $\mathbb{R}^n$. Therefore, using Lemma~\ref{lemma_rule3}, we may write 
\begin{equation}\label{eq:f_N_part}
f^L_{\mathcal{N}}(x) = \sum_{v\in\mathcal{L}}\mathds{1}_{\{x\in\mathcal{D}(v)\}}\g{}{L}{v}(x) = \g{}{L}{v_x}(x),
\end{equation}
where $v_x$ is the leaf node to which $x$ is assigned, i.e., such that $x\in\mathcal{D}(v_x)$. Taking $\arg\max$ over the left and right of \eqref{eq:f_N_part} and using the same tie-breaking procedure for both $\mathcal{N}$ and $\mathcal{S}_{\mathcal{N}}$, we have
\begin{equation} \nonumber
\begin{split}
& f_{\mathcal{N}}(x) = \underset{i \in \{1, \dots, N^L\}}{\arg\max}\{(f^L_{\mathcal{N}}(x))_i\} = \underset{i \in \{1, \dots, N^L\}}{\arg\max}\{\g{i}{L}{v_x}(x) \} \\
& \quad = f_{\mathcal{S}_{\mathcal{N}}}(x).
\end{split}
\end{equation}
\end{proof}

\subsection{Transformation Complexity}

A key advantage of the SDT transformation procedure given in Algorithm~\ref{alg_transformation} over existing approaches~\cite{nguyen2020towards, lee2019oblique, sudjianto2020unwrapping} is that our algorithm only generates essential node splits during the creation of the SDT.
As a result, the output SDT size scales polynomially in the maximum NN hidden layer width, as opposed to showing exponential scaling in the number of neurons. The complexity of $\mathcal{S}_{\mathcal{N}}$ in terms of the number of nodes is stated in Theorem~\ref{theorem_compactness}, whose proof is based on an upper bound on the number of piecewise affine regions achievable with ReLU NNs~\cite{ferlez2021bounding}.

We first define $B(x,\delta)$ as the Euclidean ball of radius $\delta$ centered at $x\in\mathbb{R}^n$, where $\delta>0$. We then recall the definitions of a hyperplane arrangement, a region of a hyperplane arrangement, and a bound on the number of regions of a hyperplane arrangement \cite{ferlez2021bounding}.

\begin{definition} \emph{(Hyperplane Arrangement).}
Let $\mathcal{A}$ be a set of affine functions, where for each $f\in\mathcal{A}$ we have $f\,:\,\mathbb{R}^n\to\mathbb{R}$. Then, $\mathcal{H}_\mathcal{A}=\bigcup_{f\in\mathcal{A}}\{x| f(x)=0\}$ is an arrangement of hyperplanes of dimension $n$.
\end{definition}

\begin{definition} \label{def_hyperplane}
\emph{(Region of a Hyperplane Arrangement).} Let $\mathcal{H}_\mathcal{A}$ be an arrangement of $N$ hyperplanes in the $n$-dimensional space, defined by a set of affine functions $\mathcal{A}$.
A non-empty open subset $R_\mathcal{A} \subseteq \mathbb{R}^n$ is called an $n$-dimensional region of $\mathcal{H}_\mathcal{A}$ if $R_\mathcal{A} = \bigcap_{f \in \mathcal{A}} \{x|f(x) \bowtie 0\}$, where $\bowtie$ is either $<$ or $>$ and there exists $\delta>0$ such that $B(x,\delta)\subseteq R_\mathcal{A}$.
The set of all regions of $\mathcal{A}$ is denoted as $\mathfrak{R}_{\mathcal{H}_\mathcal{A}}$.
\end{definition}

\setcounter{theorem}{2}
\begin{theorem} \label{theorem_hyperplane} \emph{    (Theorem 1 \cite{ferlez2021bounding}).} Let $\mathcal{H}_\mathcal{A}$ be an arrangement of $N$ hyperplanes of dimension $n$, defined by a set of affine functions $\mathcal{A}$.
The number of regions of $\mathcal{H}_\mathcal{A}$, denoted by $|\mathfrak{R}_{\mathcal{H}_\mathcal{A}}|$, is at most $\sum_{k=0}^n \binom{N}{k}$, which is bounded by $O(N^n/n!)$.
\end{theorem}

We are now ready to introduce the theorem on the complexity of the transformed SDT.

\begin{theorem} \label{theorem_compactness}
Let $\mathcal{N}$ be a NN and $\mathcal{S}_{\mathcal{N}}$ the SDT resulting from the application of Algorithm~\ref{alg_transformation} to $\mathcal{N}$. Denote the number of nodes in $\mathcal{S}_{\mathcal{N}}$ by $|\mathcal{S}_{\mathcal{N}}|$. Then, we obtain 
$$|\mathcal{S}_{\mathcal{N}}| = O\left(\prod_{l=2}^{L-1} \frac{(N^l)^{(N^1)}}{(N^1)!} 2^{N^L}\right) = O\left(\frac{2^{N^L}}{(N^1!)^{L}}\overline{l}^{N^1L}\right),$$
where $\overline{l} := \max_{l\in\{2,\dots,L-1\}}N^l$.
\end{theorem}

\begin{proof}
We introduce the notation $\mathcal{V}_l(v)$ to denote the set of nodes $v'$ such that $v'$ is a successor of $v$ and $\rho^{p(v')}(x)$ is derived via (\hyperlink{eq_condition_1}{A}) at layer $l$, while $\rho^{v'}(x)$ is not, i.e., 
\begin{equation}\nonumber
    \begin{split}
        \mathcal{V}_l(v) & = \{v' \mid \ v' \text{ is a successor node of } v \ \text{such that}: \\
        & \exists i \in \{1, \dots, N^l\}, \ \rho^{p(v')}(x) = \sigma(\g{i}{l}{p(v')}(x)), \\
        & \quad \text{and} \ \forall i' \in \{1, \dots, N^l\}, \rho^{v'}(x) \neq \sigma(\g{i'}{l}{v'}(x)) \},
    \end{split}
\end{equation}
where $\rho^{p(v')}(x) = \sigma(\g{i}{l}{p(v')}(x))$ denotes that $\rho^{p(v')}(x)$ is derived via (\hyperlink{eq_condition_1}{A}), while $\rho^{v'}(x) \neq \sigma(\g{i'}{l}{p(v')}(x))$ indicates that this is not the case for $\rho^{v'}(x)$. For example, in Fig.~\ref{fig_example}, we have $\mathcal{V}_2(v_0) = \{v_1, v_2\}$ because $\rho^{p(v_1)} = \rho^{p(v_2)} = \rho^{v_0} = \sigma(\g{1}{2}{v_0}(x))$ and node $v_1, v_2$ are derived via (\hyperlink{eq_condition_3}{C}) and (\hyperlink{eq_condition_2}{B}), respectively.

We are now ready to prove our theorem.
Given a neural network $\mathcal{N}$ and the transformed SDT $\mathcal{S}_{\mathcal{N}}$, for all nodes $v' \in \mathcal{V}_2(v_0)$, all parent nodes of $v'$ are derived via (\hyperlink{eq_condition_1}{A}) with $l=2$.
Therefore, for each $v'\in\mathcal{V}_2(v_0)$, $\mathcal{D}(v')$ is a region of the hyperplane arrangement corresponding to the set of affine functions of the form $\ell_i(x) = W^2_ix + B^2_i$ for $i\leq N^2$.
By Theorem~\ref{theorem_hyperplane}, the size of $\mathcal{V}_2(v_0)$ is at most $\sum_{k=0}^{N^1} \binom{N^2}{k}$.
Similarly, for all nodes $v' \in \mathcal{V}_2(v_0)$, the size of $\mathcal{V}_3(v')$ is at most $\sum_{k=0}^{N^1} \binom{N^3}{k}$.
Therefore, the size $|\mathcal{V}_{L-1}(v_0)|$ of the set of nodes $v$ whose parent nodes are derived via (\hyperlink{eq_condition_1}{A}) is at most $\prod_{i=2}^{L-1} \sum_{k=0}^{N^1} \binom{N^i}{k}$.
We note that all nodes $v'' \in \mathcal{V}_{L-1}(v_0)$ are not derived by (\hyperlink{eq_condition_1}{A}), while $p(v'')$ is derived by (\hyperlink{eq_condition_1}{A}). Moreover, for each node $v'' \in \mathcal{V}_{L-1}(v_0)$, there exist at most $(N^L - 1)$ pairs of $i_1$ and $i_2$ satisfying (\hyperlink{eq_condition_2}{B}).
Hence, there are at most $\prod_{i=2}^{L-1} \sum_{k=0}^{N^1} \binom{N^i}{k}\cdot 2^{N^L - 1}$ leaf nodes.
Therefore, $|\mathcal{S}_{\mathcal{N}}| \leq \prod_{i=2}^{L-1} \sum_{k=0}^{N^1} \binom{N^i}{k}\cdot 2^{N^L}$.
\end{proof}

Given a NN with hidden layers of uniform width $n$, we can also show that our procedure for creating an equivalent SDT results in a number of SDT nodes which scales polynomially with respect to the width of the NN hidden layers $n$.
\begin{lemma} \label{lemma_compactness}
Let $\mathcal{N}$ be a NN with hidden layer sizes $N^2 = N^3 = \dots = N^{L-1} = n$ and let $\mathcal{S}_{\mathcal{N}}$ be the SDT resulting from the application of Algorithm~\ref{alg_transformation} to $\mathcal{N}$.
Then, the number of nodes in $\mathcal{S}_{\mathcal{N}}$, denoted by $|\mathcal{S}_{\mathcal{N}}|$, scales polynomially with the width of the NN hidden layers $n$:
$$|\mathcal{S}_{\mathcal{N}}| = O\left(\frac{n^{(N^1)L}}{(N^1)!^{L}} 2^{N^L}\right).$$
\end{lemma}

\section{Verification Problem Formulations}\label{sec:ver_form}

\begin{figure}[t]
    \centering
    \includegraphics[width=\ifdim\columnwidth>0.6\textwidth 0.7\columnwidth \else 1.0\columnwidth \fi]{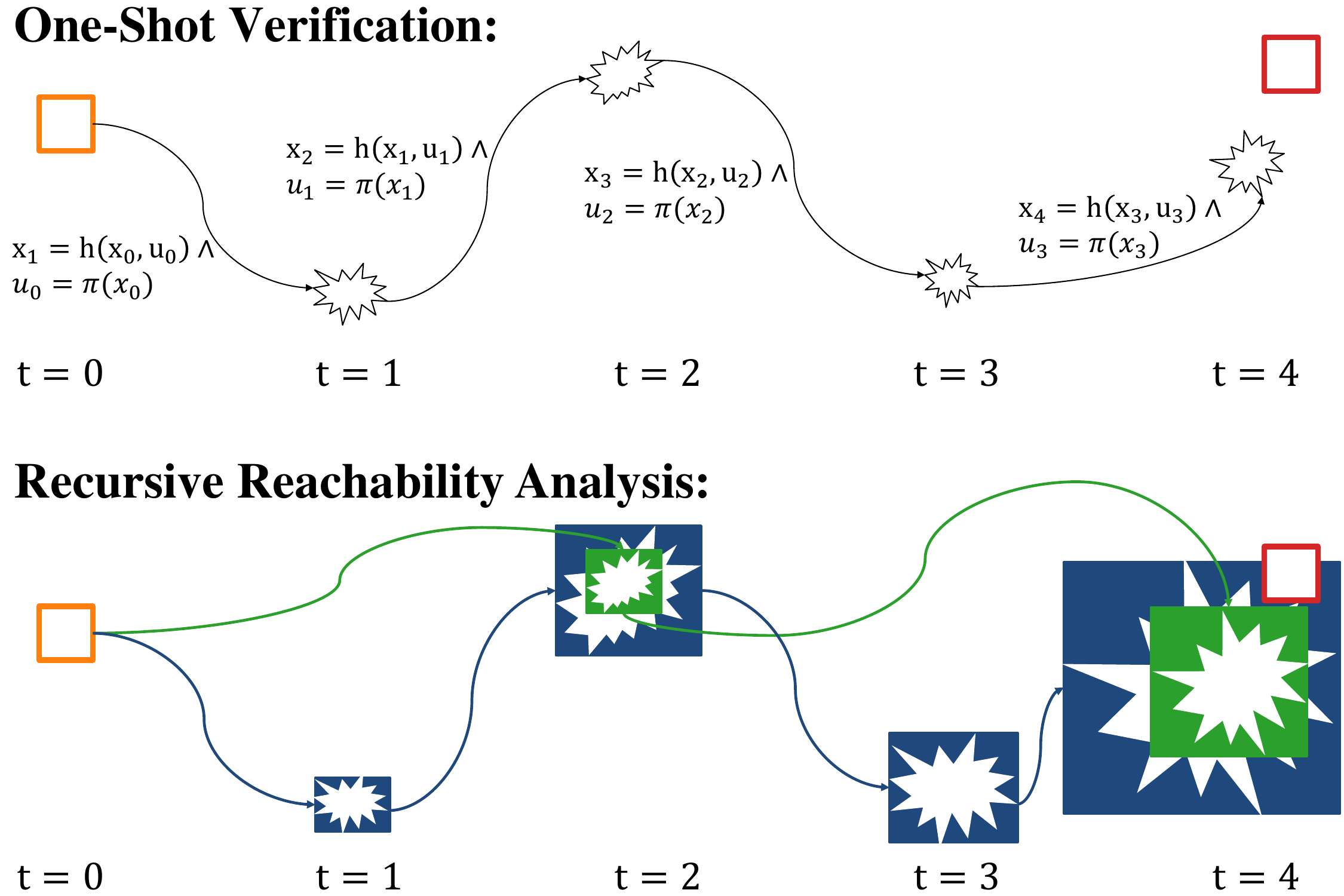}
    \caption{System-level representation, $\mathcal{X}_i$ (\protect\solidorangeline), $\mathcal{X}_g$ (\protect\solidredline), $s=1$ (\protect\solidblueline), $s=2$ (\protect\solidgreenline).}
    \label{fig_sys}
\end{figure}

We assess the impact of the proposed transformation by exploring verification problems for both SDT and NN controllers. We consider closed-loop controlled dynamical systems with state and discrete action spaces $\mathcal{X}$ and $\mathcal{U}$, respectively, and dynamics $h\,:\,\mathcal{X}\times\mathcal{U}\to\mathcal{X}$.
Let $\pi\,:\,\mathcal{X}\to\mathcal{U}$ denote a time-invariant Markovian policy (controller) mapping states to actions.
Given the system dynamics $h$, an initial set $\mathcal{X}_i\subseteq \mathcal{X}$, and goal set $\mathcal{X}_g\subseteq \mathcal{X}$, we wish to determine whether $\mathcal{X}_g$ is reachable in a finite horizon $T$ for all initial states $x_0 \in \mathcal{X}_i$ under $h$ and policy $\pi$. If so, we say that the specification $(\mathcal{X}_i,\mathcal{X}_g, T)$ is satisfied for $\pi$. 
In this context, we consider two verification approaches. 

We begin by introducing a one-shot verification methods, pictorially represented in Fig.\ref{fig_sys}, where the system dynamics over a finite horizon is encoded as a single satisfiability modulo theory (SMT) problem\cite{barrett2018satisfiability}.

\begin{problem}\textit{(One-Shot Verification).} We ``unroll" the system dynamics at time instants $t \in \{0, \ldots, T\}$ and encode the verification problem as an SMT problem using the bounded model checking approach~\cite{biere2009bounded}. Satisfaction of $(\mathcal{X}_i,\mathcal{X}_g, T)$ for a policy $\pi$ is then equivalent to showing that the following formula is \emph{not} satisfiable:
\begin{equation}
\label{eq_monolith_formula}
\begin{split}
    &\phi_1(\pi) := (x_0 \in \mathcal{X}_{i}) \wedge (x_T \not \in \mathcal{X}_g) \ \wedge \\
    &\bigwedge_{t=0}^{T-1}(x_{t + 1} = h(x_t, u_t)) \wedge (u_t = \pi(x_t)).
\end{split}
\end{equation}
\end{problem}
If $\phi_1(\pi)$ is false, then $\pi$ is guaranteed to drive the system from any $x_0\in \mathcal{X}_i$ to $\mathcal{X}_g$ within the finite horizon $T$. 

We observe that \eqref{eq_monolith_formula} requires encoding $T$ replicas of both the system dynamics $h$ and the NN policy $\pi$ in the control loop, which can make the SMT problem intractable due to its computational complexity.
Therefore, we also consider an alternative verification approach via reachability analysis.
In particular, we adopt a \emph{recursive reachability analysis}~\cite{chen2023one}, where we use the $s$-step unrolled dynamics, with $s \ll T$, to compute an over-approximation of the $s$-step reachable set in terms of a rectangle.
As shown in Fig.~\ref{fig_sys}, when setting $s = 1$, we over-approximate the reachable set at each iteration.
Alternatively, when $s = 2$, we unfold the dynamics over a 2-step horizon to mitigate the impact of the over-approximation error obtained when estimating the reachable set at $t = 1$.
The distinction between the rectangle $R_t$ for $s = 1$, in blue, and the rectangle $R_t$ for $s = 2$, in green,  visually represents the impact of adjusting the unrolling parameter on the precision of the reachable set estimation.

\begin{problem}
\textit{(Recursive Reachability Analysis (RRA)).} We fix a step parameter $s$ and recursively compute reachable sets $\mathcal{R}_{t}$, where $t=m s$ for $m\in\mathbb{N}$.
Given the system dynamics $h$, policy $\pi$, and step parameter $s$, we encode the reachable set at time $t$ as the SMT formula
\begin{equation}
\label{eq_reachability_formula}
\begin{split}
    & \mathcal{R}_t(\pi) = \bigg\{ x_t \bigg| (x_{t-s} \in \mathcal{R}_{t-s}) \;\wedge \\
    & \quad\quad\quad\quad\bigwedge_{k=t-s}^{t-1}(x_{k + 1} = h(x_k, u_k)) \wedge 
   (u_k = \pi(x_k))\bigg\}. \\
\end{split}
\end{equation}
In cases where $t$ is not divisible by the step $s$, i.e., when $t = ms + r$ with $r < s$, we calculate the reachable set $\mathcal{R}_t$ by referencing the previously computed reachable set at time $ms$, denoted as $\mathcal{R}_{ms}$. 

By setting $\mathcal{R}_0=\mathcal{X}_i$, we can verify that the specification $(\mathcal{X}_i,\mathcal{X}_g, T)$ is satisfied by checking satisfaction of the following SMT formula  
\begin{equation} \nonumber
\begin{split}
& \phi_{\text{RRA}}(\pi) := (x_0 \in \mathcal{X}_i)\wedge (x_T \not \in \mathcal{X}_g)\ \wedge \\
& \bigwedge_{m=1}^{\lfloor T/s\rfloor}(x_{ms} \in \mathcal{R}_{ms}(\pi))\ \wedge \\
& \bigwedge_{t=\lfloor T/s \rfloor s}^{T-1}(x_{t + 1} = h(x_t, u_t)) \wedge (u_t = \pi(x_t)).\\
\end{split}
\end{equation}
\end{problem}
While RRA may yield overly conservative reachable sets due to error propagation, it is often more tractable than one-shot verification, since it decomposes the overall reachability problem into a set of smaller sub-problems, each having a finite horizon of $s$ and encoding the NN policy only $s$ times, with $s \ll T$.

\section{Case Studies} \label{sec_experiment}

We evaluate the proposed transformation on three environments from the OpenAI Gym library~\cite{brockman2016openai}, as detailed below.

\subsection{The Environments} \label{sec_env}

\begin{figure}[t]
    \centering
    \includegraphics[width=\ifdim\columnwidth>0.6\textwidth 0.5\columnwidth \else 0.6\columnwidth \fi]{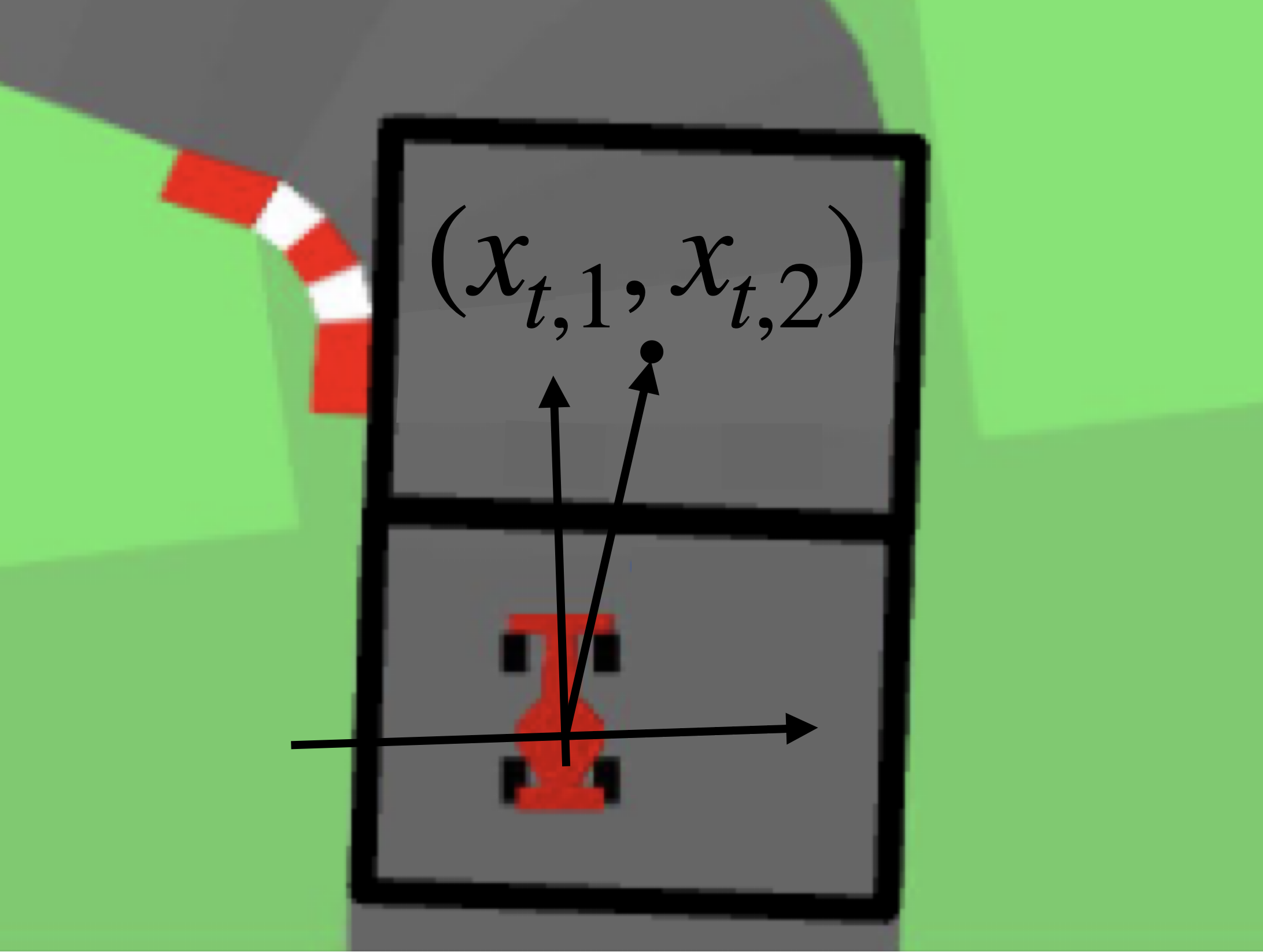}
    \caption{An illustration of the car racing problem. The tuple $(x_{t,2},x_{t,3})$ denotes the coordinates of the center of the tile ahead of the car, relative to the axes centered on the car as shown.}
    \label{fig_car_racing}
\end{figure}

\subsubsection{MountainCar-v0} \label{sec_mountain_car}

In the MountainCar control task, an underpowered car needs to reach the top of a hill starting from a valley within a fixed time horizon $T$~\cite{moore1990efficient}.
The state vector $x_t \in \mathbb{R}^2$ comprises the car's position $x_{t,1}$ and horizontal velocity $x_{t,2}$ at time step $t$, while the action $u_t \in \{-1, 0, 1\}$ represents the car's acceleration action, either left (L), idle (I), or right (R), respectively. The system dynamics is detailed in Appendix~\ref{app_mountaincar_dynamics}. The objective is to reach the flag as quickly as possible, and the agent incurs a penalty of $-1$ for each time step taken.

While the reference NNs are trained to reach the top of the hill in the minimum number of steps, we seek to verify whether a given controller will reach the goal state within $T$ steps, starting from any point in an initial interval.
Our specification $(\mathcal{X}_i,\mathcal{X}_g,T)$ is given by
\begin{equation}\nonumber
\begin{split}
  & \mathcal{X}_{i} = \{x_{0, 1}\in [\mathfrak{i}_l ,\mathfrak{i}_u],x_{0, 1}= 0 \},\,\,\mathcal{X}_{g} =  \{x_{T, 1} \geq \mathfrak{g}\}, 
\end{split}
\end{equation}
where parameters $\mathfrak{i}_l$ and $\mathfrak{i}_u$ denote upper and lower bounds on the initial state, respectively, and $\mathfrak{g}$ denotes the goal coordinates.

\subsubsection{CartPole-v1} \label{sec_cart_pole}

In the CartPole control task~\cite{kumar2024empirical}, a pole is attached to a cart moving along a frictionless track, to be balanced upright for as long as possible within a fixed horizon $T$.
The state vector $x_t\in\mathbb{R}^4$ consists of the cart position $x_{t,1}$, the cart velocity $x_{t,2}$, the pole angle $x_{t,3}$, and the pole angular velocity $x_{t,4}$ at time step $t$. The action $u_t \in \{0, 1\}$ denotes the acceleration of the cart to the left or right. The system dynamics is detailed in Appendix~\ref{app_cartpole_dynamics}.
A reward of $1$ is given for every time step in which the pole is maintained in an upright position.
Additionally, the time horizon is truncated to a maximum of 500 steps, as this duration is empirically shown to be sufficient to demonstrate the agent's effectiveness in performing the task~\cite{kumar2024empirical}.

We seek to verify that the pole is balanced upright within tolerance $\mathfrak{g}$ at the end of horizon $T$, starting from a range of initial cart positions and pole angles.
With $\mathfrak{i}_l$, $\mathfrak{g}$ as parameters, our specification $(\mathcal{X}_i,\mathcal{X}_g,T)$ is given by
\begin{equation}\nonumber
\begin{split}
& \mathcal{X}_{i} = \{x_{0,1}\in [\mathfrak{i}_l,0], x_{0,3} \in [\mathfrak{i}_l,0], x_{0,2} = x_{0,4} = 0\}\\
& \mathcal{X}_{g} = \{ |x_{T,3}| \geq \mathfrak{g}\}.
\end{split}
\end{equation}

\subsubsection{CarRacing-v2} \label{sec_car_racing}

The CarRacing-v2 environment is a  simulation environment in which the state is represented by a $96 \times 96$-pixel image in RGB format. Rather than directly processing high-dimensional image inputs that may lead to NN that cannot be efficiently transformed, our purpose is to explore the feasibility of extracting a succinct but meaningful set of state features from image-based environments in an autonomous driving context, and then transforming the associated NN controllers into SDTs. At each time step $t$, we extract the state vector $x_t \in \mathbb{R}^3$ from the environment as shown in Fig.~\ref{fig_car_racing}.
The components of the state vector are defined as follows: $x_{t,1}$ represents the car's velocity, while $(x_{t,2}, x_{t,3})$ denotes the relative coordinates of the center of the second tile ahead within the current point of view. 
These coordinates are acquired directly from the simulator.
The action $u_t$ corresponds to the car's action, which can be idle (I), left (L), right (R), gas (G), or brake (B), respectively. 
Simultaneous combinations of multiple actions at a step are not allowed.

The agent's objective in this environment is to navigate a race track successfully while minimizing the completion time. The reward system operates as follows: for each frame, the agent receives a reward of $-0.1$. Additionally, the agent is awarded a bonus reward of $1000/T$ for each track tile it visits, where $T$ corresponds to the total number of tiles present throughout the entire track.

\begin{figure*}[t]
    \centering
    \begin{subfigure}{0.32\textwidth}
    \includegraphics[width=\linewidth]{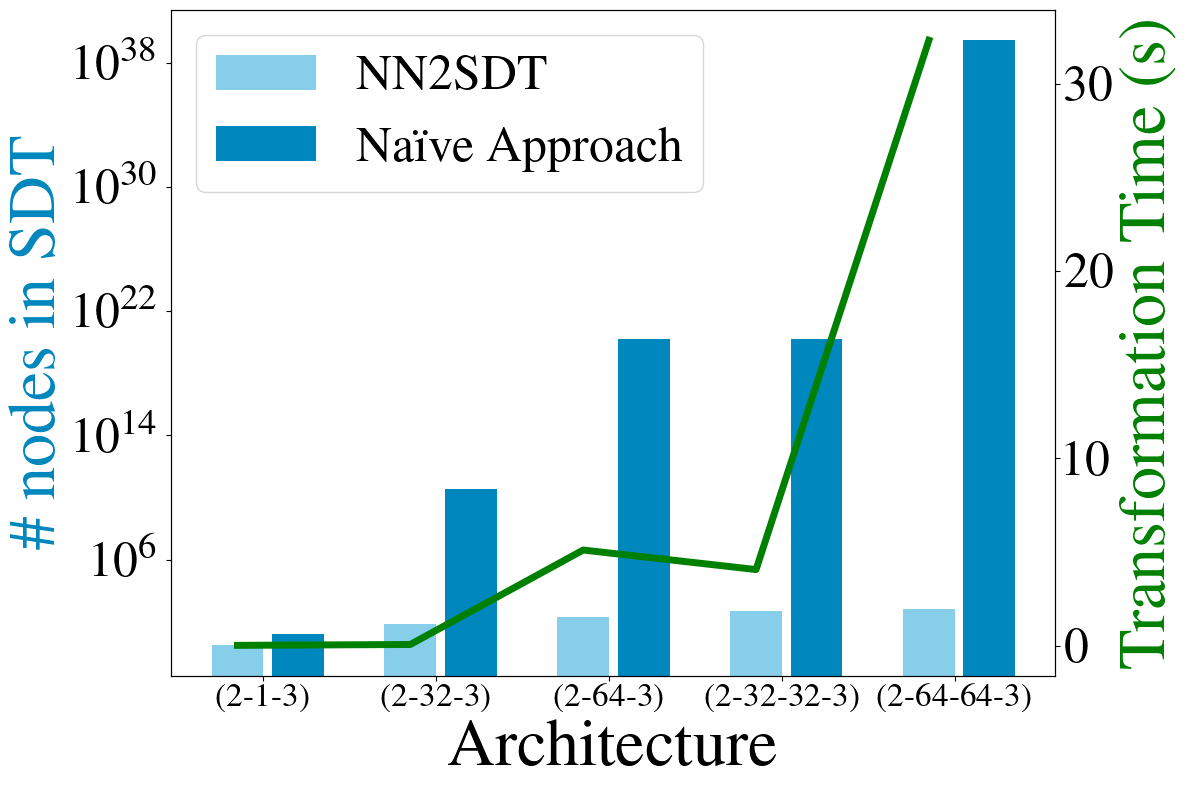}
    \caption{MountainCar}
    \end{subfigure}~
    \begin{subfigure}{0.32\linewidth}
    \includegraphics[width=\linewidth]{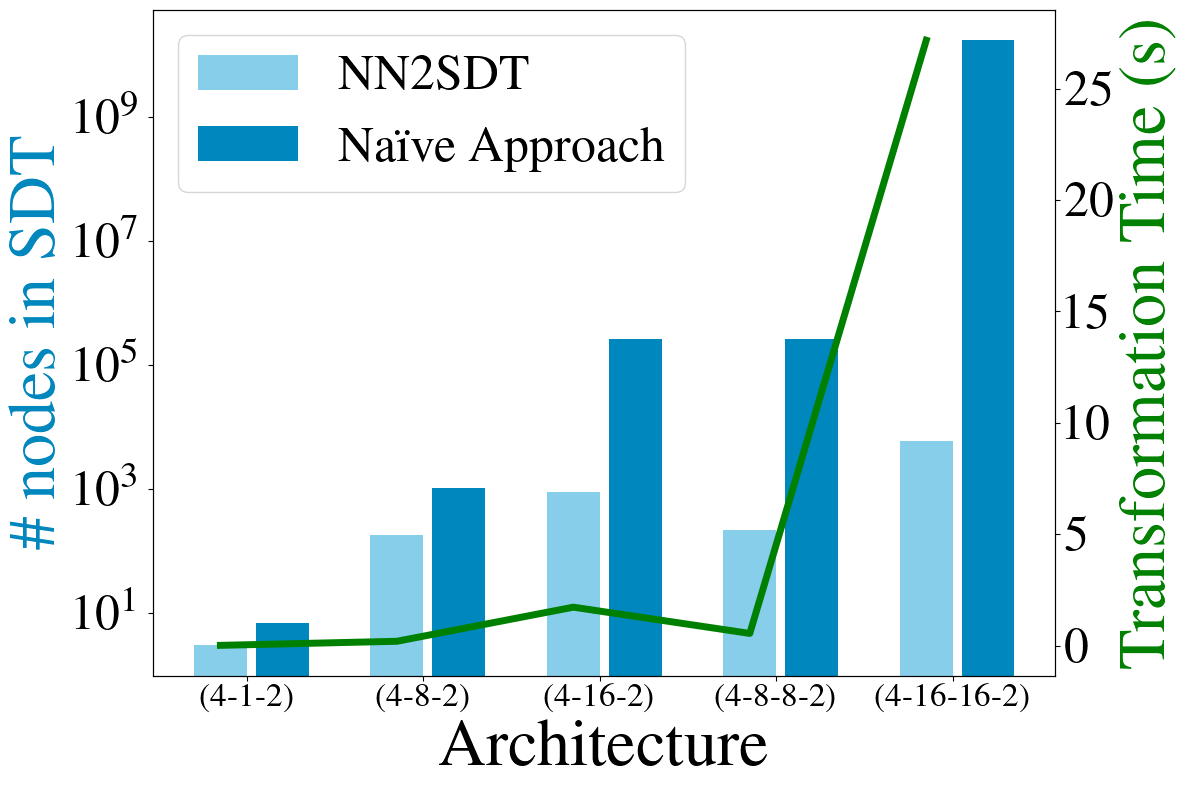}
    \caption{CartPole}
    \end{subfigure}~
    \begin{subfigure}{0.32\linewidth}
    \includegraphics[width=\linewidth]{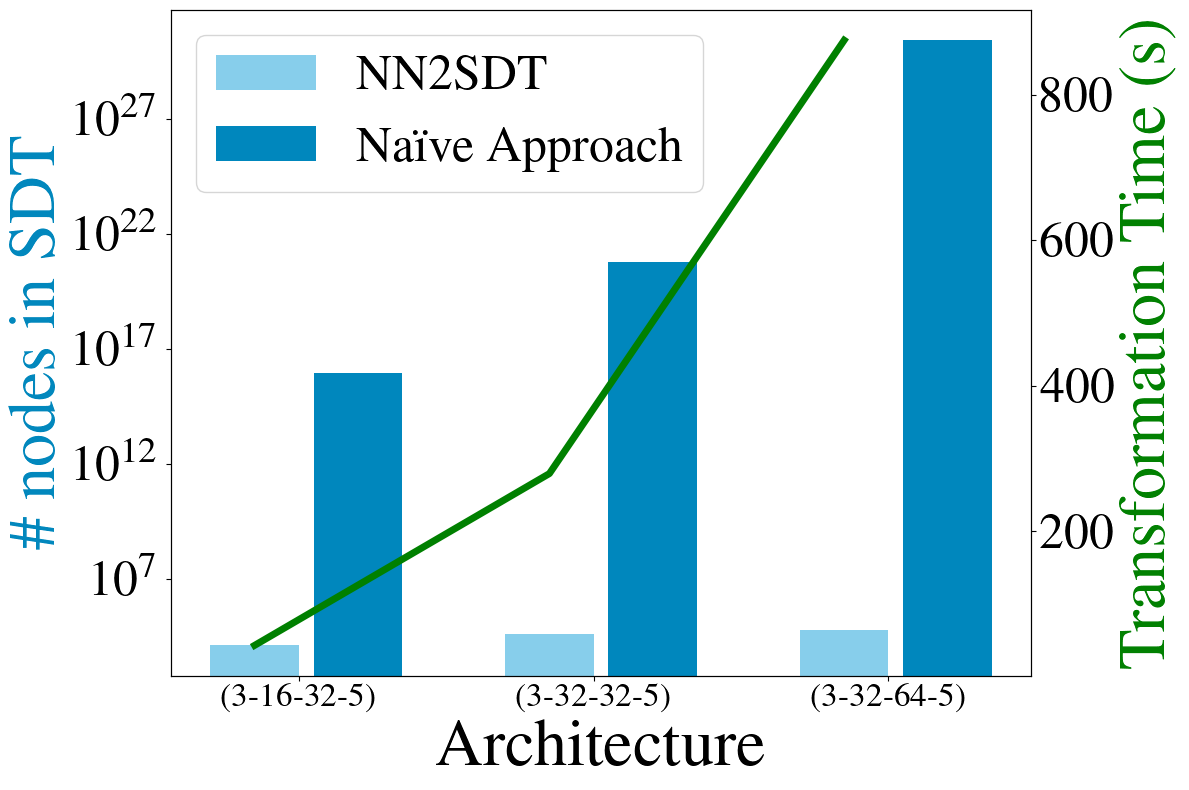}
    \caption{CarRacing}
    \end{subfigure}
    \caption{Size of the transformed SDT and transformation time. The number of tree nodes under the na\"{i}ve approach is calculated by $2^{\sum_{i=2}^{L} N^i} - 1$.}
    \label{fig_transformation}
\end{figure*}

\subsection{Reference NN Controllers}

\begin{table}
\caption{Hyperparameters in OpenAI Gym.}
\label{tbl_parameter}
\begin{center}
\begin{tabular}{|c | c | c | c |}
\hline
    Environment& $\mathfrak{i}_l$&$\mathfrak{i}_u$&$\mathfrak{g}$ \\
\hline
    MountainCar &$-0.11$&$-0.1$&$0.5$\\
\hline
    CartPole &$-0.1$ &N/A&$\pi/15$\\
\hline 
\end{tabular}
\end{center}
\end{table}

\begin{table}
\caption{Average performance of reference controllers over 100 episodes.}
\label{tbl_performance}
\begin{center}
    \scriptsize
\begin{tabular}{|c | c | c |}
\hline
Environment& $\mathbf{N}$ & Avg. Reward\\
\hline
\multirow{5}{*}{MountainCar}
        &$(2,1,3)$ & -126.06 \\
        &$(2,32,3)$ & -134.33 \\
        &$(2,64,3)$ & -154.76 \\
        &$(2,32,32,3)$ & -122.08 \\
        &$(2,64,64,3)$ & -142.57 \\
\hline
\multirow{5}{*}{CartPole}
        & $(4,1,2)$ &$43.56$\\
        & $(4,8,2)$ &$256.03$\\
        & $(4,16,2)$ &$125.05$\\
        & $(4,8,8,2)$ &$500.00$\\
        & $(4,16,16,2)$ &$169.06$\\
\hline 
\multirow{4}{*}{CarRacing}
        & Convolutional DQN & $761.35$ \\
        & $(3,16,32,5)$ &$647.78$ \\
        & $(3,32,32,5)$ &$854.49$ \\
        & $(3,32,64,5)$ &$419.79$ \\
\hline
\end{tabular}
\end{center}
\end{table}

We adopt the problem setups available in the open-source package OpenAI Gym~\cite{brockman2016openai} and the parameters in Table~\ref{tbl_parameter}.
We begin by training reference NN controllers over $5 \times 10^6$, $2 \times 10^6$, and $2 \times 10^7$ steps for the MountainCar, CartPole, and CarRacing problems, respectively.
The average performance scores from this training are shown in Table~\ref{tbl_performance}.
In Table~\ref{tbl_performance}, ``Convolutional DQN" refers to the default convolutional Deep Q-Network available in the stable-baselines framework \cite{raffin2021stable}.
An interesting observation is that the state controller utilizing a three-dimensional state vector $x_t \in \mathbb{R}^3$ achieves higher scores than the Convolutional DQN.
This indicates that utilizing a three-dimensional state vector is effective for driving control in the CarRacing scenario.
Motivated by this finding, we proceed to apply our proposed transformation technique to construct equivalent SDTs.

\subsection{Evaluation of the Transformation Procedure}

Figure ~\ref{fig_transformation} illustrates the transformation time and the number of nodes of the SDT corresponding to each NN controller.
We also display the estimated number of nodes under a na\"{i}ve transformation method that produces a decision tree that includes a leaf node for every possible combination of choices determined by the 
ReLU activation regimes of the NN inner layers as well as the orderings of the output layer values. That is, each path from the root to a leaf node consists of branches determined by inner layer ReLU functions and comparisons between output layer values.  
The number of tree nodes created under this method is $2^{\sum_{i=2}^{L} N^i} - 1$, where $\sum_{i=2}^{L} N^i$ gives the total number of NN inner and output layer neurons. 
In Fig.~\ref{fig_ex_NN}, we observe that $2^{\sum_{i=2}^{L} N^i} - 1 = 2^4 - 1 = 15$.
As shown in Fig.~\ref{fig_transformation}, the percentage of size reduction for the SDT increases with the number of NN layers. 
Furthermore, the size of the transformed SDT for the CartPole problem increases faster than that for the MountainCar problem, which can be attributed to the difference in $N^1 = |x_t|$. 
Nevertheless, we were able to transform all NN controllers in under $60$ seconds for the MountainCar and CartPole problems, and under $900$ seconds for the CarRacing environment, while existing methods failed with NNs exceeding ten neurons.

\subsection{Verification Results}

\begin{figure*}[t]
    \centering
    \includegraphics[width=1.0\textwidth]{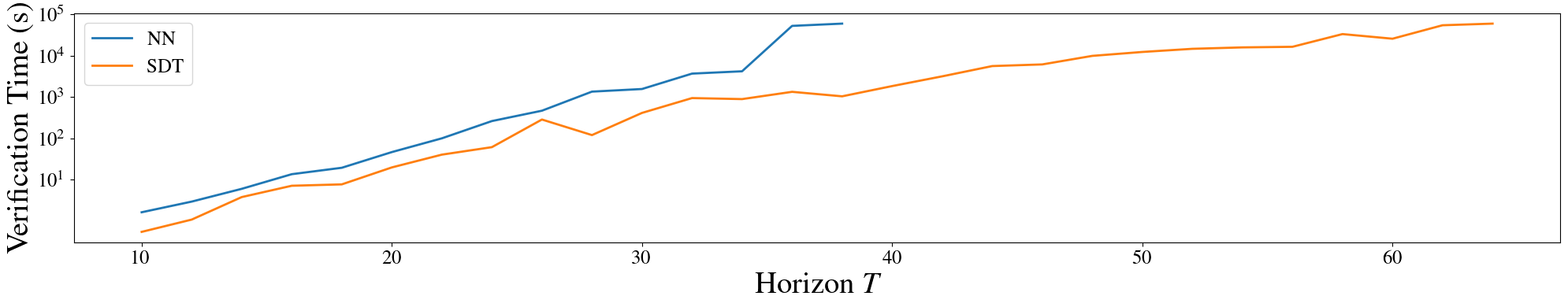}
    \caption{One-shot verification times for the MountainCar NN controller with $\mathbf{N} = (2, 1, 3)$ and the transformed SDT controller.}
    \label{fig_verification}
\end{figure*}

\begin{figure*}[t]
    \centering
    \begin{subfigure}{1\textwidth}
        \centering
        \includegraphics[width=1.0\linewidth]{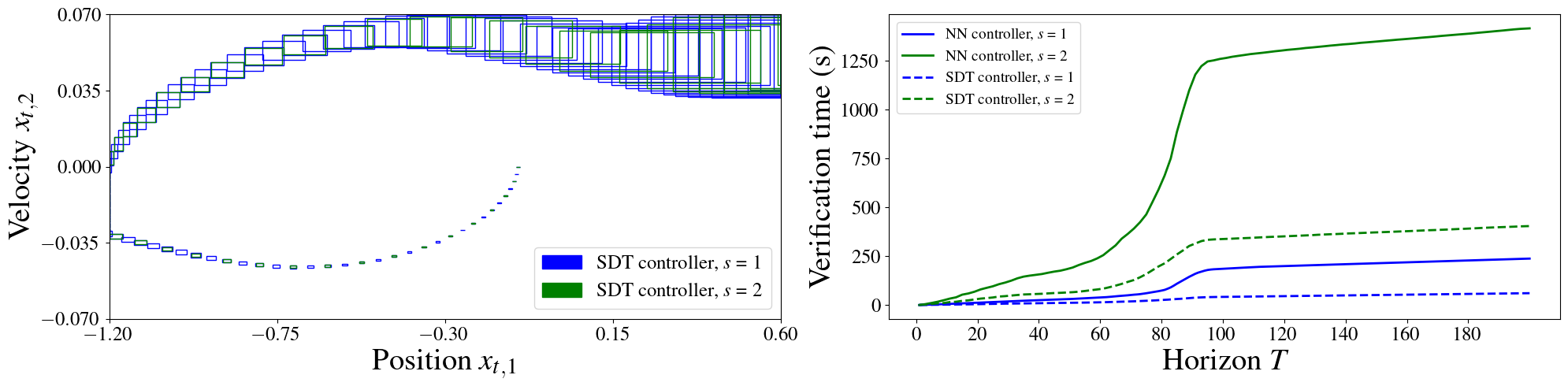}
        \caption{Reachable sets and verification time for MountainCar NN with $\mathbf{N} = (2, 32, 32, 3)$.}
        \label{fig_reachability_verification_mountain_car}
    \end{subfigure}
    \begin{subfigure}{1\textwidth}
        \centering
        \includegraphics[width=1.0\linewidth]{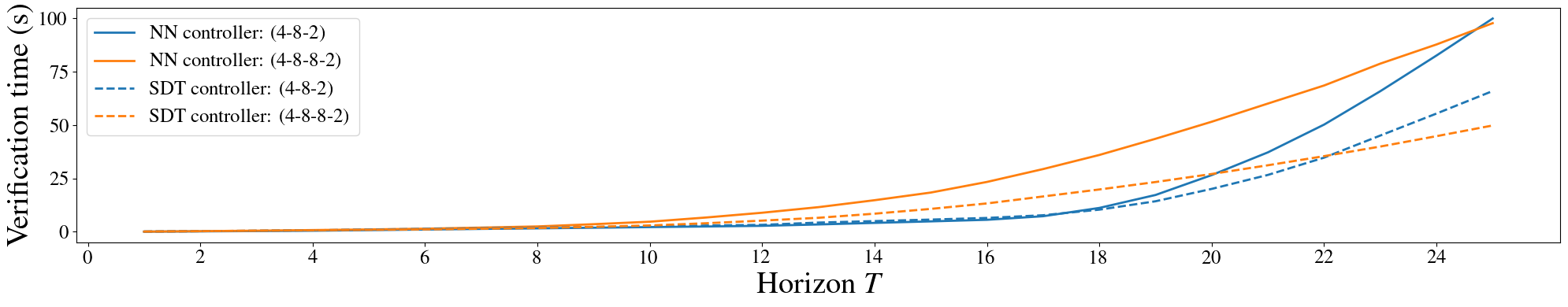}
        \caption{Verification time for CartPole.}
        \label{fig_reachability_verification_cart_pole}
    \end{subfigure}
    \caption{Recursive reachability analysis results.}
\end{figure*}

A primary reason for transforming NN controllers to SDT controllers is to explore the reduction in computational burden for solving the SMT problems involved in the verification formulations presented in Section \ref{sec:ver_form}. All SMT problems in our experiments are solved using an in-house implementation of a satisfiability modulo convex programming (SMC) solver~\cite{shoukry2018smc}, which integrates the Z3 satisfiability (SAT) solver~\cite{de2008z3} with Gurobi~\cite{achterberg2019s}.
Beginning with the one-shot verification approach, Fig.~\ref{fig_verification} shows a comparison of the verification times for $\mathcal{N}$ with $\mathbf{N} = (2, 1, 3)$ and $\mathcal{S}_{\mathcal{N}}$ with three leaf nodes in the MountainCar problem. The verification time for the NN increases significantly compared to the equivalent SDT, over two orders of magnitude. However, verification takes more than 60000 seconds for $\mathcal{N}$ when $T>38$ and for $\mathcal{S}_{\mathcal{N}}$ when $T > 64$. For the CartPole problem, all controllers fail to solve the verification problem \eqref{eq_monolith_formula} within 60000 seconds for $T>10$, highlighting the challenges associated with one-shot verification.

We then evaluate the efficiency of RRA for both NN and SDT controllers. Figure~\ref{fig_reachability_verification_mountain_car} plots the reachable sets and the overall verification time for an example MountainCar controller with $s \in \{1,2\}$ in \eqref{eq_reachability_formula}. As discussed in Section~\ref{sec:ver_form}, the area of the estimated reachable set for $s=1$ exceeds that obtained for $s=2$. RRA verification enables the verification of properties in a system with a larger horizon compared to one-shot verification. Figure~\ref{fig_reachability_verification_cart_pole} shows the verification for a collection of CartPole controllers with $s=1$ in \eqref{eq_reachability_formula}. RRA verification becomes more efficient using the equivalent SDT controllers. 

\begin{figure*}[t]
    \centering
    \begin{subfigure}{\ifdim\columnwidth>0.6\textwidth 0.45\columnwidth \else 0.9\columnwidth \fi}
        \includegraphics[width=\linewidth]{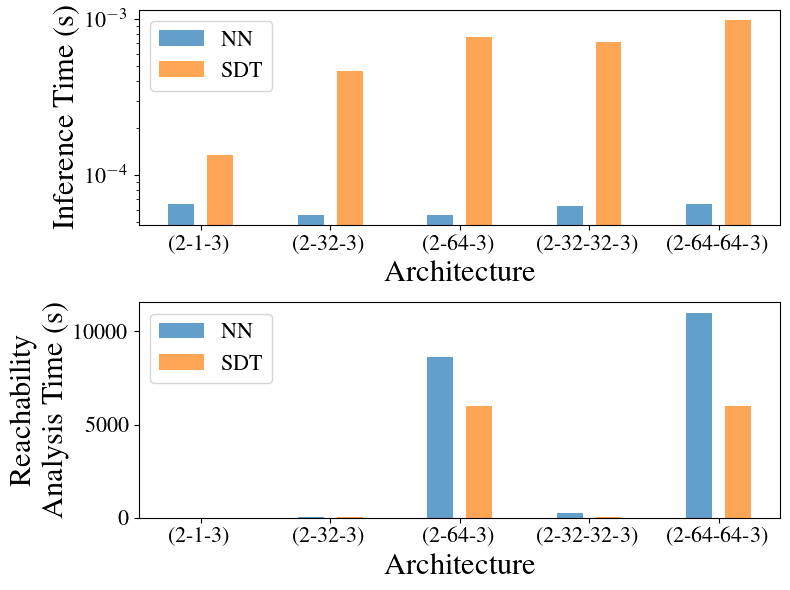}
        \caption{MountainCar.}
    \end{subfigure}
    \begin{subfigure}{\ifdim\columnwidth>0.6\textwidth 0.45\columnwidth \else 0.9\columnwidth \fi}
        \includegraphics[width=\linewidth]{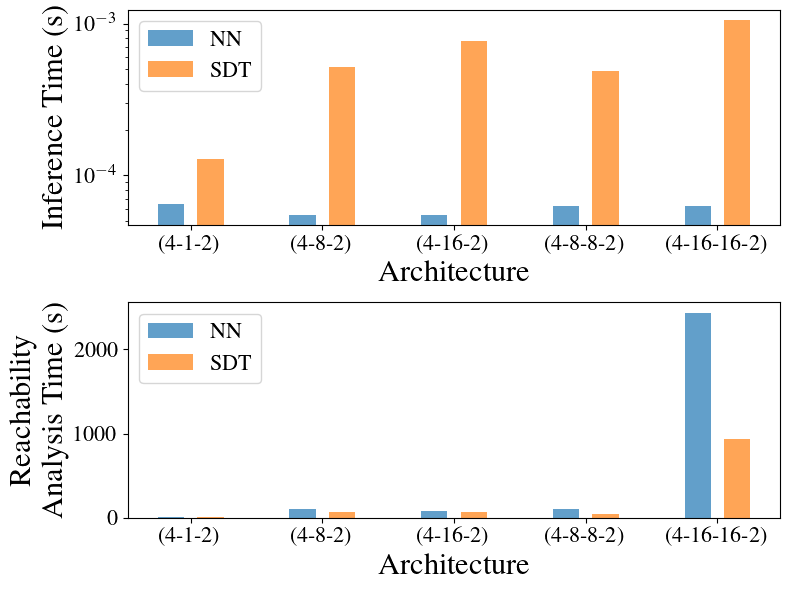}
        \caption{CartPole.}
    \end{subfigure}
    \caption{Inference and RRA runtimes for NN and SDT controllers. For RRA, $T=200$ for MountainCar and $T=25$ for CartPole. NN outperforms SDT in inference time by an order of magnitude on a log scale, while SDT accelerates verification for the MountainCar and CartPole problems by up to $21\times$ and $2\times$, respectively.}
    \label{fig_overall}
\end{figure*}

Figure~\ref{fig_overall} compares the average inference and overall RRA verification times for a collection of NN and equivalent SDT controllers.
We set $T=200$ for the MountainCar problem and $T=25$ for the CartPole problem.
As shown, the difference in inference time increases with the number of layers in the NN.
This behavior is attributed to the capability of NNs to undergo parallel hardware acceleration.
Our results show that the transformed SDT controller accelerates verification for the MountainCar problem by up to $21 \times$ and for the CartPole problem by up to $2\times$.
Moreover, the larger the number of layers in the NN, the larger the difference between runtime for NN verification and runtime for SDT verification.

\begin{figure*}[t]
    \centering
    \begin{subfigure}{\ifdim\columnwidth>0.6\textwidth 0.45\columnwidth \else 0.9\columnwidth \fi}
        \includegraphics[width=\linewidth]{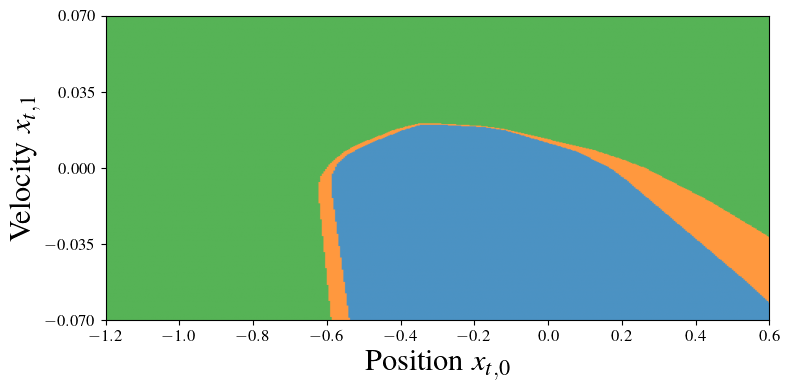}
        \caption{$\mathbf{N} = (2, 64, 3)$}
    \end{subfigure}
    \begin{subfigure}{\ifdim\columnwidth>0.6\textwidth 0.45\columnwidth \else 0.9\columnwidth \fi}
        \includegraphics[width=\linewidth]{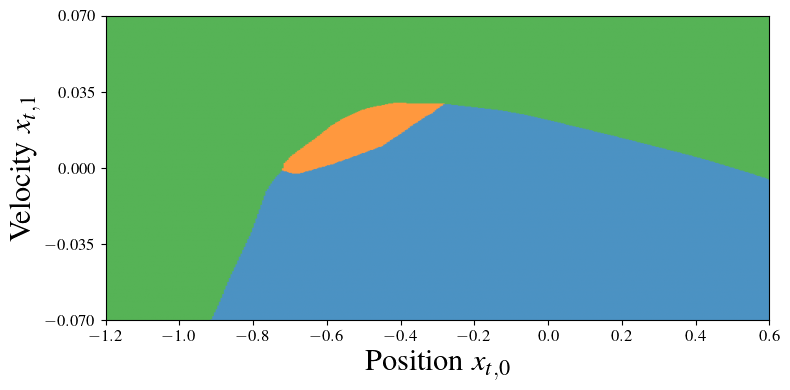}
        \caption{$\mathbf{N} = (2, 64, 64, 3)$}
    \end{subfigure}
    \caption{MountainCar NN policies. Legend: L (\protect\solidblueline), I (\protect\solidorangeline), R (\protect\solidgreenline).}
    \label{fig_action_mountain_car}
\end{figure*}

\begin{figure*}[t]
    \centering
    \begin{subfigure}{\ifdim\columnwidth>0.6\textwidth 0.45\columnwidth \else 0.9\columnwidth \fi}
        \includegraphics[width=\linewidth]{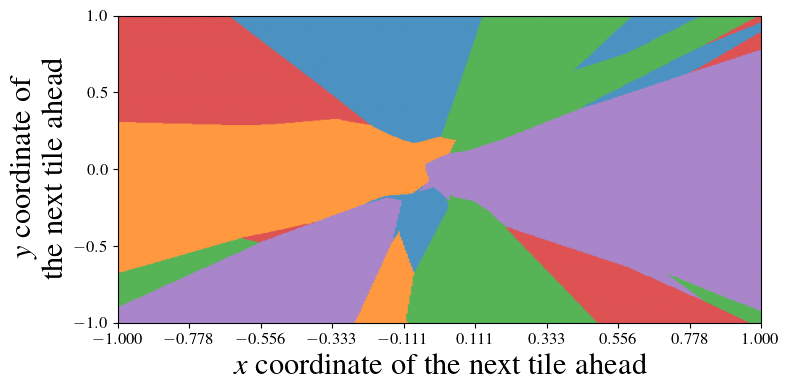}
        \caption{$\begin{array}{l}
                x_{t,1} = 50~m/s
        \end{array}$}
    \end{subfigure}
    \begin{subfigure}{\ifdim\columnwidth>0.6\textwidth 0.45\columnwidth \else 0.9\columnwidth \fi}
        \includegraphics[width=\linewidth]{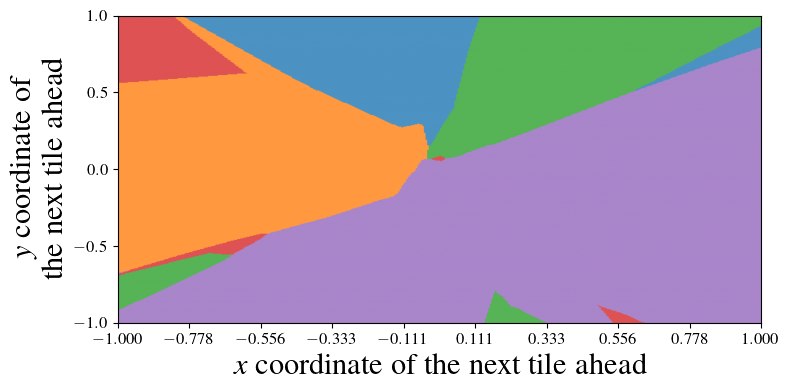}
        \caption{$\begin{array}{l}
                x_{t,1} = 100~m/s
        \end{array}$}
    \end{subfigure}
    \caption{CarRacing NN policies for $\mathbf{N} = (3, 32, 32, 5)$. Legend: I (\protect\solidblueline), L (\protect\solidorangeline), R (\protect\solidgreenline), G (\protect\solidredline), B (\protect\solidpurpleline).}
    \label{fig_action_car_racing}
\end{figure*}

An advantage of SDT controllers is that the policy can be inferred from the tree structure.
Figure~\ref{fig_action_mountain_car} shows the policies and SDT leaf node partitions for example MountainCar controllers.
We observe that constant action state-space regions become more complex as the depth of the NN increases.
Figure~\ref{fig_action_car_racing} plots the policies and leaf node partitions of the CarRacing controller with $\mathbf{N} = (3, 32, 32, 5)$.
The observed behavior indicates that when the vehicle is traveling at high speed and the relative position of the second tile ahead is on the left, the controller instructs the vehicle to turn left.
Conversely, if the relative position of the second tile ahead is on the right, the controller advises the vehicle to turn right.
Finally, when the vehicle is not traveling at high speed, the controller's strategy is to accelerate when the second tile ahead is further away, presumably to optimize its position and speed.

\section{Conclusion}

We proposed a cost-effective algorithm to construct equivalent SDT  controllers from any discrete-output argmax-based NN controller. Empirical evaluation of formal verification tasks performed on two benchmark OpenAI Gym environments shows a significant reduction in verification times for SDT controllers. Future work includes experimenting with more sophisticated RL environments and performing more extensive comparisons with state-of-the-art verification tools. 

\section*{References}
\bibliographystyle{IEEEtran}
\bibliography{root}

\appendices

\section{Notation} \label{app_notation}
We summarize the notation used in this paper in Table~\ref{tbl_notation}.

\begin{table}[H]
\caption{Notation adopted in the paper.}
\label{tbl_notation}
\begin{tabularx}{\columnwidth}{@{}p{0.3\columnwidth}X@{}}
    \textbf{Notation} & \textbf{Definition} \\
    \toprule
    \textbf{General} \\
    $A_{i,j}$ & Element at row $i$ and column $j$ in matrix $A$ \\
    $a_i$ or $(a)_i$ & $i$-th element of vector $a$ \\
    $\mathds{1}_{x \in D}$ & Indicator function \\ \\
    \textbf{Neural Network} \\
    $\mathcal{N}$ & Neural network \\
    $L$ & Total number of layers \\
    $N^l$ & Number of neurons in the $l$-th layer \\
    $\mathbf{N}$ & Vector of neuron counts $(N^1, \ldots, N^L)$ \\
    $W^l, B^l$ & Weights and biases in the $l$-th layer \\
    $\alpha$ & Element-wise ReLU activation function \\
    $f_F^l(x)$ & Feedforward function for the $l$-th layer \\
    $f_\mathcal{N}^l(x)$ & Characteristic function at the $l$-th layer of $\mathcal{N}$ \\
    $f_\mathcal{N}(x)$ & Characteristic function of neural network $\mathcal{N}$ \\ \\
    \textbf{Soft Decision Tree} \\
    $\mathcal{S}$ & Soft decision tree \\
    $\mathcal{I}_{\mathcal{S}}, \mathcal{L}_{\mathcal{S}}$ & Sets of inner and leaf nodes for $\mathcal{S}$ \\
    $v$ & Node \\
    $w^v, b^v$ & Weights and biases associated with inner node $v$ \\
    $\rho^v(x)$ & Probability threshold for inner node $v$ \\
    $Q^v$ & Distribution associated with leaf node $v$ \\
    $l(v), r(v)$ & Left and right child of inner node $v$ \\
    $p(v)$ & Parent of node $v$ \\
    $f_\mathcal{S}^v(x)$ & Characteristic function of node $v$ \\
    $f_\mathcal{S}(x)$ & Characteristic function of soft decision tree $\mathcal{S}$ \\
    $\mathcal{D}(v)$ & Effective domain of node $v$ \\ \\
    \textbf{Transformation} \\
    $\mathcal{S}_\mathcal{N}$ & Transformed $\mathcal{S}$ from $\mathcal{N}$ \\
    $\bot$ & Undefined value \\
    $\g{i}{l}{v}(x)$ & Pre-activation function providing the output of the $i$-th neuron in the $l$-th layer of the neural network as a function of input in $\mathcal{D}(v)$ prior to the ReLU \\
    $\overline{\g{i}{l}{v}(x)}$ & Post-activation function providing the output of the $i$-th neuron in the $l$-th layer of the neural network as a function of input in $\mathcal{D}(v)$ after the ReLU \\ \\ 
    \textbf{Complexity Analysis} \\
    $\mathcal{A}$ & Set of affine functions \\
    $\mathcal{H}_\mathcal{A}$ & Arrangement of hyperplanes \\
    $R_\mathcal{A}$ & Region of $\mathcal{H}_\mathcal{A}$ \\
    $\mathfrak{R}_{\mathcal{H}_\mathcal{A}}$ & Set of all regions of $\mathcal{A}$ \\
    $\mathcal{V}_l(v)$ & Set of successor nodes of $v$ derived via (\hyperlink{eq_condition_1}{A}) at layer $l$ while the parent is not \\
    \bottomrule
\end{tabularx}
\end{table}

\section{System Dynamics Used in the Case Studies} 

We outline the system dynamics used for the MountainCar and CartPole case studies, introduced in Section~\ref{sec_env}.

\subsection{MountainCar} \label{app_mountaincar_dynamics}

The system dynamics is given by 
\begin{equation}
\label{eq_mountain_car}
\begin{split}
h(x_t, u_t) = 
 \left( \begin{array}{l} x_{t,1} + x_{t,2} \\ x_{t,2} + \alpha (u_t - 1) - \beta \cos(\gamma x_{t,1}) \\ \end{array} \right)
\end{split}, 
\end{equation}
where the state vector $x_t \in \mathbb{R}^2$ comprises the car's position $x_{t,1}$ and horizontal velocity $x_{t,2}$ at time step $t$.
The action $u_t \in \{-1, 0, 1\}$ represents the car's acceleration, which can be left (L), idle (I), or right (R). 
The parameters $\alpha$, $\beta$, and $\gamma$ correspond to the magnitude of the acceleration, the time period, and the slope of the mountain, respectively.

The dynamics $h$ in \eqref{eq_mountain_car} involves a non-linear cosine function, which is non-convex and cannot be directly supported by an SMC solver.
We then approximate the cosine function using a piecewise linear function. 
Specifically, since the cosine function is concave over the interval $[-\frac{\pi}{2}, \frac{\pi}{2}],$ we partition this interval uniformly into multiple segments.
For each segment, with left endpoint $x_l$ and right endpoint $x_r$ such that $\frac{-\pi}{2} \leq x_l \leq x_r \leq \frac{\pi}{2}$, we employ linear lower and upper bounds to approximate the cosine function as follows:
\begin{equation}
\nonumber
\begin{split}
    & \cos(x_l) + (x - x_l) \frac{\cos(x_r) - \cos(x_l)}{x_r - x_l} \leq \cos(x) \leq \\
    &\quad\quad \cos\left(\frac{x_l + x_r}{2}\right) - \left(x - \left(\frac{x_l + x_r}{2}\right)\right)\sin\left(\frac{x_l + x_r}{2}\right),
\end{split}
\end{equation}

We use an analogous approximation over the input range $[\frac{\pi}{2}, \frac{3\pi}{2}]$, where the cosine function is convex.

\subsection{CartPole} \label{app_cartpole_dynamics}

The system dynamics is given by
\begin{equation}
\label{eq_cart_pole}
\begin{split}
& x_{t+1} = h(x_t, u_t)\\
& = \left( 
\begin{array}{l}
    x_{t,1} + \Delta_t x_{t,1} \\
    x_{t,2} + \Delta_t \frac{f(2u_t - 1) + m_p l (x_{t,4}^2\sin x_{t,3} - \dot{x}_{t,4}^2\cos x_{t,3})}{m_p + m_c} \\
    x_{t,3} + \Delta_t x_{t,4} \\
    x_{t,4} + \Delta_t \dot{x}_{t,4} \\
\end{array} \right) \\
& \dot{x}_{t,4} = \frac{g\sin x_{t,3} - \cos x_{t,3} \left(\frac{f(2u_t - 1) + m_p l (x_{t,4}^2\sin x_{t,3})}{m_p + m_c}\right)}{l \left(\frac{4}{3} - \frac{m_p\cos^2 x_{t,3}}{m_p + m_c}\right)}, \\
\end{split}, 
\end{equation}
where $f$ denotes the magnitude of acceleration, $m_p$ and $m_c$ represent the masses of the pole and the cart respectively, $l$ stands for the length of the pole, and $\Delta_t$ indicates the time step length.
The state vector $x_t\in\mathbb{R}^4$ consists of the cart position $x_{t,1}$, cart velocity $x_{t,2}$, pole angle $x_{t,3}$, and pole angular velocity $x_{t,4}$ at time step $t$.
The action $u_t \in \{0, 1\}$ specifies whether the cart accelerates to the left or right.

The dynamics $h$ in \eqref{eq_cart_pole} also involves a non-linear cosine function similar to \eqref{eq_mountain_car}, which is non-convex and cannot be directly supported by an SMC solver.
We then approximate the cosine function using a piecewise linear function detailed in Appendix~\ref{app_mountaincar_dynamics}. 

\begin{IEEEbiography}[{\includegraphics[width=1in,height=1.25in,clip,keepaspectratio]{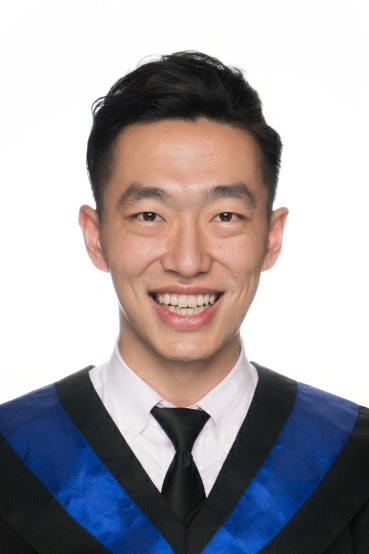}}]{Kevin Chang}
Kevin Chang joined the Department of Electrical and Computer Engineering at the University of Southern California in Los Angeles as a PhD candidate.
He holds a B.S. in Electrical Engineering from National Taiwan University.
His research specializes in machine learning, verification, and closed-loop control, with a particular emphasis on applications in neural networks.
\end{IEEEbiography}
\begin{IEEEbiography}[{\includegraphics[width=1in,height=1.25in,clip,keepaspectratio]{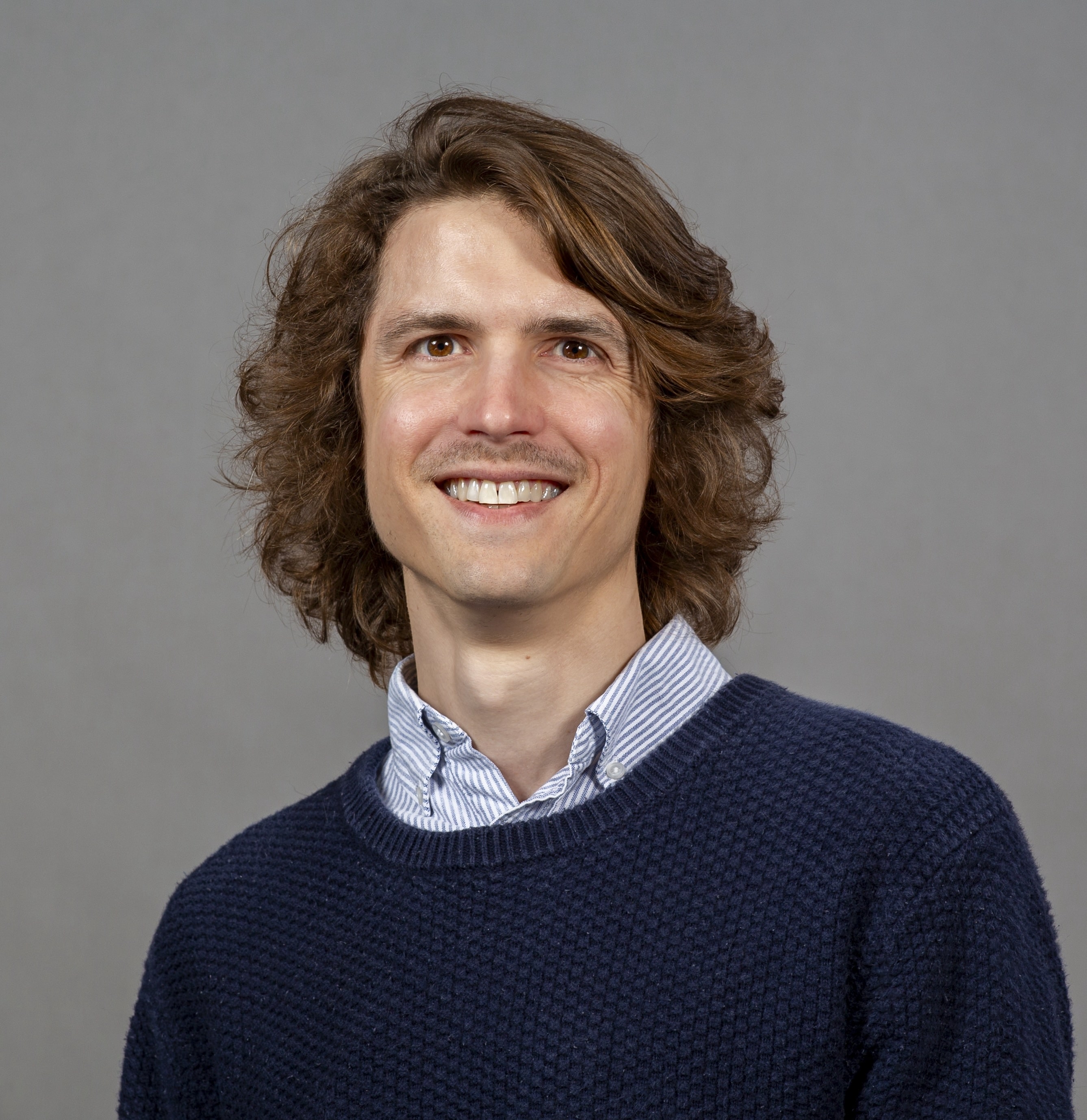}}]{Nathan Dahlin}
Nathan Dahlin joined the Department of Electrical and Computer Engineering at the University at Albany as an Assistant Professor in Fall 2023.
He holds a B.S., M.S., and Ph.D. in electrical engineering and an M.A. in applied mathematics from the University of Southern California.
Prior to joining the University at Albany, he was a Postdoctoral Research Associate at the University of Illinois Urbana-Champaign.
From 2008 to 2015, Nathan was a senior audio DSP research and development engineer at Audyssey Laboratories in Los Angeles.
His research is centered on machine learning, stochastic control, optimization, and microeconomics, with particular focus on applications in smart energy systems.
\end{IEEEbiography}

\begin{IEEEbiography}[{\includegraphics[width=1in,height=1.25in,clip,keepaspectratio]{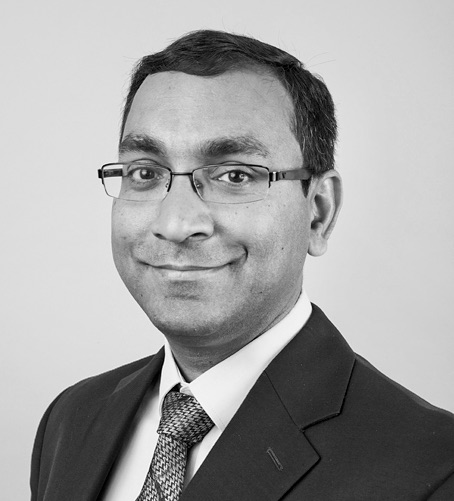}}]{Rahul Jain} (Senior Member, IEEE), is a Professor of Electrical and Computer Engineering and Computer Science at the University of Southern California (USC), a  Director of the USC Center for Autonomy and Artificial Intelligence, and a visiting research scientist at Google. He received his Ph.D. (EECS) from the University of California, Berkeley.  He has received numerous awards including the NSF CAREER award, the ONR Young Investigator award, and an IBM Faculty award. His interests span reinforcement learning, stochastic control, statistical learning, and game theory.
\end{IEEEbiography}

\begin{IEEEbiography}[{\includegraphics[width=1in,height=1.25in,clip,keepaspectratio]{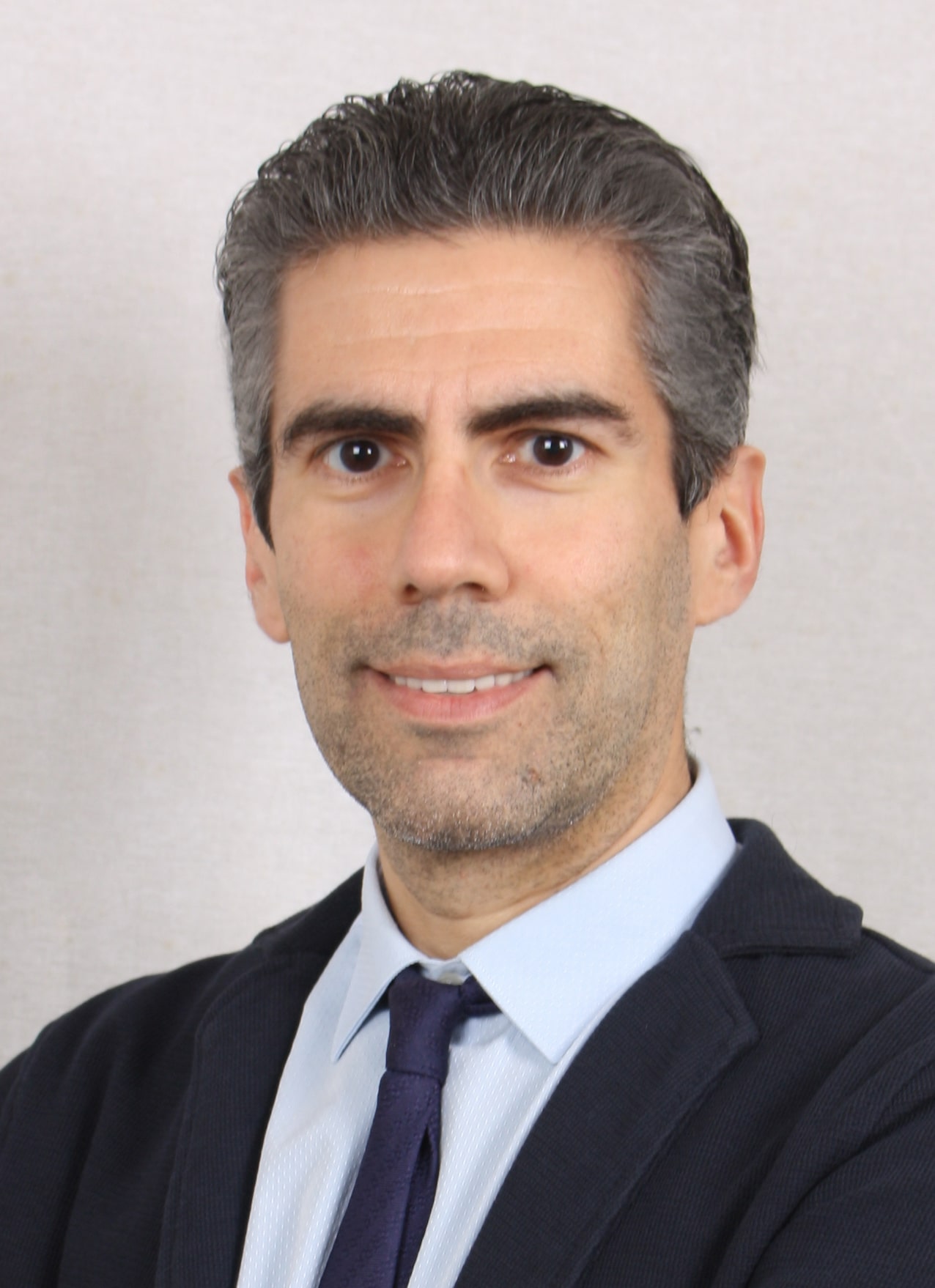}}]{Pierluigi Nuzzo} (Senior Member, IEEE) is an Associate Professor in the Department of Electrical Engineering and Computer Sciences at the University of California (UC), Berkeley, and an Adjunct Associate Professor at the University
of Southern California (USC), Los Angeles. Before joining UC Berkeley, he was the Kenneth C. Dahlberg Chair and Associate Professor of Electrical and Computer Engineering and Computer Science at USC, and a Co-Director of the USC Center for Autonomy and Artificial Intelligence (AI). He received
the Ph.D. in Electrical Engineering and Computer
Sciences from UC, Berkeley,
and B.S. and M.S. degrees in electrical and computer engineering from the University of Pisa and
the Sant’Anna School of Advanced Studies, Pisa,
Italy. His interests revolve around methodologies and tools for high-assurance design of cyber-physical systems (CPSs) and systems-on-chip, including the application of formal methods and optimization theory to problems in CPSs, electronic design automation (EDA), autonomy, security, and AI. His awards include the NSF CAREER Award, the DARPA Young Faculty Award, the Early-Career Awards from the IEEE Council on EDA and the Technical Committee on CPSs, the Okawa Research Grant, the UC Berkeley EECS David J. Sakrison Memorial Prize, and several best paper and design competition awards.
\end{IEEEbiography}

\end{document}